\documentclass[10pt,journal,compsoc]{IEEEtran}

\usepackage[american]{babel}
\usepackage[utf8]{inputenc}
\usepackage[T1]{fontenc}
\usepackage{amsmath,amsthm,amssymb}
\usepackage{mathtools}
\usepackage{mathrsfs}
\usepackage{fca}
\usepackage{dsfont}  
\usepackage{cancel}
\usepackage[colorlinks,citecolor=blue,urlcolor=black,hidelinks,linktocpage]{hyperref}
\usepackage[all]{hypcap}
\usepackage{booktabs}
\usepackage{cleveref}
\let\cref\Cref
\usepackage{ marvosym }
\newtheorem{thm}{Theorem}[section]
\newtheorem{corollary}[thm]{Corollary}
\newtheorem{proposition}[thm]{Propostion}
\newtheorem{lemma}[thm]{Lemma}
\newtheorem{definition}[thm]{Definition}
\newtheorem{problem}[thm]{Problem}

\theoremstyle{definition}

\theoremstyle{remark}
\newtheorem{remarks}[thm]{Remark}

\usepackage{nicefrac}

\renewcommand{\epsilon}{\varepsilon}
\renewcommand{\phi}{\varphi}
\newcommand\logeq{\mathrel{\vcentcolon\Leftrightarrow}}
\DeclareMathOperator{\Th}{Th}
\DeclareMathOperator{\Ext}{Ext}
\DeclareMathOperator{\Int}{Int}
\DeclareMathOperator{\conf}{conf}
\DeclareMathOperator{\supp}{sup}
\newcommand{\N}{\mathbb{N}}
\newcommand{\Scon}{\mathbb{S}}
\newcommand{\Tcon}{\mathbb{T}}
\newcommand{\Ocon}{\mathbb{O}}

\usepackage{adjustbox}
\usepackage{tikz}
\usetikzlibrary{positioning}
\usepackage{enumerate}

\usepackage[shrink=25,babel,kerning=true]{microtype}
\usepackage{paralist}
\usepackage{todonotes}
\presetkeys{todonotes}{color=blue!5}{}

\usepackage[ruled,vlined,linesnumbered]{algorithm2e}
\usepackage{algorithmic}

\usepackage[style=ieee,sorting=nyt,backend=bibtex,doi=false,isbn=false,url=false]{biblatex}
\usepackage{graphicx}

\usepackage{fixltx2e}
\usepackage{stfloats}

\usepackage{url}

\newcommand{\pkcore}{$pq$-core\ }
\newcommand{\pkcorenw}{$pq$-core}

\newcommand{\pkcores}{$pq$-cores\ }
\newcommand{\pkcoresnw}{$pq$-cores}
\newcommand{\pkleq}[1]{\leq_{#1}}
\DeclarePairedDelimiter{\abs}{\lvert}{\rvert}

\begin{document}
%
\title{Knowledge Cores in Large Formal Contexts}
%
%
%
%

\author{Tom~Hanika, Johannes~Hirth
  \IEEEcompsocitemizethanks{\IEEEcompsocthanksitem T.~Hanika \& J.~Hirth were
    with the Knowledge \& Data Engineering Group, Dep.\ of Electrical Engineering and Computer Science, University of Kassel.\protect\\
    E-mail: tom.hanika@cs.uni-kassel.de, hirth@cs.uni-kassel.de.\protect\\
    Authors are given in alphabetical order.  No priority in authorship is
    implied.}
  \thanks{Manuscript received xx, xxxx; revised xx, xxxx.}}

%
%

\markboth{Journal of XXXX,~Vol.~XX, No.~X, X~2020}%
{Shell \MakeLowercase{\textit{et al.}}: Bare Advanced Demo of IEEEtran.cls for IEEE Computer Society Journals}
%



\IEEEtitleabstractindextext{%
\begin{abstract}
  Knowledge computation tasks are often infeasible for large data sets. This is
  in particular true when deriving knowledge bases in formal concept analysis
  (FCA). Hence, it is essential to come up with techniques to cope with this
  problem. Many successful methods are based on random processes to reduce the
  size of the investigated data set. This, however, makes them hardly
  interpretable with respect to the discovered knowledge. Other approaches
  restrict themselves to highly supported subsets and omit rare and interesting
  patterns. An essentially different approach is used in network science, called
  $k$-cores. These are able to reflect rare patterns if they are well connected
  in the data set. In this work, we study $k$-cores in the realm of FCA by
  exploiting the natural correspondence to bi-partite graphs.  This structurally
  motivated approach leads to a comprehensible extraction of knowledge cores
  from large formal contexts data sets.
\end{abstract}

}

\maketitle

\IEEEdisplaynontitleabstractindextext

%
\IEEEpeerreviewmaketitle

\ifCLASSOPTIONcompsoc
\IEEEraisesectionheading{\section{Introduction}\label{sec:introduction}}
\else
\section{Introduction}
\label{sec:introduction}
\fi

Large (binary) relational data sets are a demanding challenge for contemporary
knowledge discovery methods using formal concept analysis~\cite{fca-book}. This
is due to the fact that many considered problems in this realm are
computationally intractable, e.g., enumerating formal concepts, i.e., closed
sets, or computing the canonical
base~\cite{kuznetsov2004icd,journals/dam/DistelS11} of the implicational
theory. A cause is the potentially exponential large output size of knowledge
discovery processes. For example, large knowledge bases may be incomprehensible
to human readers. Different methods were developed to adapt to the growth of
data sets. Sophisticated algorithms employ filtering for data reduction. For
example, formal concepts can be filtered by their support in the data set. This
is done in Apriori like techniques~\cite{1401887,stumme2002efficient}. More
recent methods consider the minimum description
length~\cite{cispa2924}. However, all these approaches are unable to cope with
large relational data sets for two reasons: first, they cannot discover rare
combinations of attributes that are (comparatively) highly supported in the data
set; secondly, computations require an infeasible amount of steps. Moreover,
random approaches do not succeed either in these cases, since low supported
combinations are unlikely to be sampled. Other techniques, such as feature
combination or object clustering \cite{conf/cla/CodocedoTA11,
  journals/eswa/AswanikumarS10} lack in meaningfulness.

In general, there are two approaches to overcome the requirements of
large data sets with respect to knowledge discovery. One line of
research is to introduce novel knowledge features apart from closed
sets and their related notions. This may lead to results which are not
accessible to well studied knowledge procedures, e.g., from formal
concept analysis. The other well investigated practice is to develop
data reduction procedures that reduce the data sets significantly. For
example, latent semantic analysis or unsupervised clustering of
attributes~\cite{conf/cla/CodocedoTA11, journals/eswa/AswanikumarS10}
is often applied. This, however, does often lead to
unexplainable features.

Here we step in by translating a graph theoretic notion for data set reduction,
i.e., $k$-Cores by Seidman~\cite{Seidman1983}, to the realm of formal concept
analysis. The inviolable constraint for our investigation is to maintain
interpretability as well as explainability of knowledge with respect to the
original data set. To this end we study theoretically as well as experimentally
the impact of the core reduction process on the conceptual knowledge. Using this we
demonstrate a principle method to discover \emph{interesting cores of knowledge}
in large data sets. In detail, we give a formal overview of to be defined
\mbox{\pkcores} and their reduction effects on conceptual structures and implicational
theories. Furthermore, we provide valuations for choosing interesting cores in
large relational data sets.

We complement our findings by introducing knowledge transformation
algorithms. For a given data set and an initial \pkcore they are able to provide
a computationally efficient navigation process in the emerging knowledge
structure of all \pkcores. Finally, we argue that our methods are able to cope
with arbitrary subsets of binary relational data.

The rest of our work is structured as follows. In~\cref{thefca} we first
recollect common notations from formal concept analysis and introduce cores in
formal contexts thereafter in~\cref{sec:coresfca}. The related formal concept
lattice and canonical base are investigated in~\cref{sec:lattice}
and~\cref{sec:canbase}. This is followed by an extensive experimental study
in~\cref{coreChar} and \cref{larger} which is concluded by a presentation of efficient algorithms
for \pkcores in~\cref{algorithms}. After a discussion of related work
in~\cref{relwork} we conclude with~\cref{sec:conclusion}.

\section{Formal Concept Analysis}
\label{thefca}
Formal concept analysis (FCA) deals with binary relational data
sets~\cite{Wille1982, fca-book}. These are represented in \emph{formal context}
$(G,M,I)$ where the finite sets $G$ and $M$ are called \emph{objects} and
\emph{attributes}, respectively. The binary relation $I$ between these sets is
called \emph{incidence}, where $(g,m) \in I$ is interpreted as ``object $g$ has
attribute $m$''. Two derivation operators emerge on the power sets of $G$ and
$M$: $\cdot’:\mathcal{P}(G)\to\mathcal{P}(M)$ where $A\mapsto A'\coloneqq\{m\in M
\mid \forall g \in A: (g,m)\in I\}$ and $\cdot’:\mathcal{P}(M)\to\mathcal{P}(G)$
dually. Composing the two operators leads to two \emph{closure operators} (i.e.,
idempotent, monotone, and extensive maps) on $\mathcal{P}(G)$ and
$\mathcal{P}(M)$. We investigate in this work \emph{induced sub-contexts}, i.e.,
$\mathbb{S}=(H,N,J)$ with $H\subseteq G$, $N\subseteq M$, and $J=I\cap (H\times
N)$, denoted by $\mathbb{S}\leq \context$. When multiple formal contexts are in
play we often use the incidence relation for indicating a derivation, e.g.,
$\{g\}^{I}$ for a derivation of $g\in G$ in $\context$ and $\{g\}^{J}$ for a
derivation of $g\in H$ in $\mathbb{S}$. A \emph{formal concept} is a pair
$(A,B)\in\mathcal{P}(G)\times \mathcal{P}(M)$ with $A'=B$ and $A = B'$. We call
$A$ the \emph{extent} and $B$ the \emph{intent} of $(A,B)$ and denote with
$\Ext(\context)$ and $\Int(\context)$ the sets of all extents and intents
respectively. The set of all formal concepts of $\context$ is denoted by 
$\mathfrak{B}(\context)$. This set can be ordered by $\leq$ where $(A,B)\leq
(C,D)\logeq A\subseteq C $ for $(A,B),(C,D)\in\mathfrak{B}(\context)$. The
ordered set of all formal concepts is denoted by $\BV(\context)$. The
fundamental theorem of FCA states that $\BV(\context)$ is a (complete) lattice.
Furthermore, we investigate \emph{implications} in this work, i.e., $A\to B$,
where $A,B\subseteq M$. We say $A\to B$ is valid iff $B'\subseteq A'$. The set
of all valid implications is denoted by $\Th(\context)$. Usually, one does work
with a base of the theory, e.g., Duquenne–Guigues-Base~\cite{implbase}
(\emph{canonical base}), denoted by $\mathcal{C}_{\context}$. It can be computed
using \emph{pseudo-intents}, i.e., $P\subseteq M$ with $P \neq P''$ and
$Q''\subsetneq P$ holds for every \mbox{pseudo-intent} $Q\subsetneq P$. The
recursive nature of this definition is by design. Despite beeing the minimal
base of of the implications from $\Th(\context)$, the set of all
{pseudo-intents} can still be exponential in the size of the
context~\cite{kuznetsov2004icd}.

\subsection{Cores in Formal Contexts}
\label{sec:coresfca}
Our theory on \pkcores is based on bipartite cores 
by~\citeauthor{conf/apvis/AhmedBFHMM07} \cite[Section
3.1]{conf/apvis/AhmedBFHMM07}.  We
translated their approach to the realm of formal concept analysis, exploiting the natural
correspondence between bipartite graphs and formal contexts. This results in the
following definition.

\begin{definition}
  Let $\context=(G,M,I),\mathbb{S}=(H,N,J)$ be formal contexts with
  $\mathbb{S}\leq\mathbb{K}$. We call $\mathbb{S}$ a $pq$-\emph{core} of
  $\context$ for $p,q\in \mathbb{N}$, iff
  \begin{itemize}[i)]
  \item $\mathbb{S}$ is $pq$-\emph{dense}, i.e.,\\ \phantom{a} \hfill $\forall g
    \in H, \forall m \in N: \abs{\{g\}^J} \geq p\wedge\abs{\{m\}^J} \geq q$
  \item $\mathbb{S}$ is \emph{maximal}, i.e., \\ \phantom{a} \hfill
    $\nexists O\leq\context: \mathbb{O}\ pq\text{-dense}\wedge \Scon\neq \mathbb{O}\wedge \Scon\leq \mathbb{O}$
  \end{itemize}

\end{definition}

We denote this by $\mathbb{S}\leq_{p,q}\context$. In particular we call
contexts $\mathbb{S}$ with $\mathbb{S}\leq_{0,q}\context$ an
\emph{attribute-core} and $\mathbb{S}\leq_{p,0}\context$ an \emph{object-core}.

\begin{proposition}[Uniqueness]
  Let $\context$ be a formal context and $p,q\in \mathbb{N}$. Then there
  exists only one $S\leq \context$ with $S\pkleq{p,q} \context$.
\end{proposition}

\begin{proof}
  Let $\mathbb{S}=(H,N,J)$ and $\mathbb{T}=(U,V,L)$ be two different
  formal contexts with $\mathbb{S}\leq\context$ and
  $\mathbb{T}\leq\context$. Furthermore, for some $p,q\in\mathbb{N}$
  we have that $\mathbb{S}{\leq_{p,q}}\context$ and
  $T{\leq_{p,q}}\context$. Construct the context
  $\mathbb{D}=(H\cup U,N\cup V, J\cup L)$. Then it follows
  that
  \[ \forall g \in H\cup U, \forall m \in N\cup V: \abs{\{g\}^{J\cup L}} \geq p
    \wedge \ \abs{\{m\}^{J\cup L}} \geq q \] Hence, $\mathbb{D}$ is $pq$-dense
  and a $\mathbb{S},\mathbb{T}$ are proper sub-contexts of $\mathbb{D}$. This
  contradicts the maximality of $\mathbb{S}$ and $\mathbb{T}$.
\end{proof}

Based on this result we refer to $\mathbb{S}\,{\leq_{pq}}\,\mathbb{K}$ as
\emph{the} \pkcorenw. We depict the formal context of an example \pkcore
in~\cref{fig:example-core}. On the left is the formal context of the prominent
``Living beings and Water'' example from~\cite{fca-book} and on the
right is the $4,3$-core of it. We observe that the objects ``Bean'' and
``Leech'' as well as the attributes ``suckles its offspring'' and ``two seed
leafs'' are removed. Even though $|\{\text{Bean}\}’|\geq 4$ it is removed by a
cascading effect triggered by the removal of the attribute ``two seed leaves''.

\begin{figure}[t]
  \begin{minipage}[t]{0.54\linewidth}
    \vspace*{-\dimexpr\baselineskip\relax}
    \adjustbox{trim=6ex 0ex,scale=0.65}{\begin{cxt}
    \cxtName{}
    \att{1}
    \att{2}
    \att{3}
    \att{4}
    \att{5}
    \att{6}
    \att{7}
    \att{\textcolor{red}8}
    \att{\textcolor{red}9}
    \obj{...XXX..X}{\textcolor{red}1}
    \obj{xxx..x...}{2}
    \obj{xx.x.x.X.}{3}
    \obj{xxxx.x...}{4}
    \obj{X.X..X...}{\textcolor{red}5}
    \obj{...xxxx..}{6}
    \obj{..xxxxx..}{7}
    \obj{..x.xxx..}{8}
  \end{cxt}}
\end{minipage}\hskip0pt\hfill\hskip0pt
\begin{minipage}[t]{0.46\linewidth}
  \vspace*{-\dimexpr\baselineskip\relax}
  \adjustbox{trim=5ex 0ex,scale=0.65}{
    \begin{cxt}
    \cxtName{}
    \att{1}
    \att{2}
    \att{3}
    \att{4}
    \att{5}
    \att{6}
    \att{7}
    \obj{xxx..x.}{2}
    \obj{xx.x.x.}{3}
    \obj{xxxx.x.}{4}
    \obj{...xxxx}{6}
    \obj{..xxxxx}{7}
    \obj{..x.xxx}{8}
  \end{cxt}}
\end{minipage}
\begin{tikzpicture}[baseline, anchor=north]
  \node (c) at (0,-0.1) [inner sep=0pt, align=left,text width=\linewidth,execute at begin node=\setlength{\baselineskip}{8.5pt}] {\scriptsize
  \textbf{Attributes}: 1.~Can move around, 2.~has limbs, 3.~lives in water,
  4.~lives on land, 5.~needs chlorophyll, 6.~needs water, 7.~one seed
  leafs, 8.~suckles its offspring, 9.~two seed leafs; \textbf{Objects}: 1.~Bean, 2.~Bream, 3.~Dog, 4.~Frog, 5.~Leech, 6.~Maize, 7.~Reed, 8.~Spike-weed};
  \end{tikzpicture}
  \label{fig:example-core}
  \caption{\emph{Living Creatures and Water} context (l) and it's $4,3$-core (r)}
\end{figure}

\section{Concept Lattices of $pq$-Cores }
\label{sec:lattice}
In this section we investigate the relation of the concept lattice for a \pkcore
to the concept lattice of the originating formal context. We investigate in
particular the influence of the parameters $p$ and $q$. The computation of the
\pkcore for some $p,q$ can be understood as a sequential removal of objects and
attributes in arbitrary order. Based on this observation we analyze the impact
of object and attribute removal on concept lattices. To this end, we first take
a look at a proposition about structural embeddings. For some
$X\subseteq\mathfrak{B}(\context)$ we use the notation $\bigvee X$ for the
\emph{supremum} of $X$ in $\BV(\context)$ and $\bigwedge X$ for the
\emph{infimum} of $X$ in $\BV(\context)$,  cf.~\cite{fca-book}.

\begin{proposition}[{\cite[Proposition~31 on page~98]{fca-book}}]\ \\
  \label{prop:embed}
  Let $\context = (G,M,I)$, $\mathbb{T} = (U,M,L)$, and $\mathbb{S} = (G,N,J)$,
  be formal contexts with $\mathbb{T}\leq\context$ and
  $\mathbb{S}\leq\context$. Then the mapping $\BV(\mathbb{T})\rightarrow
  \BV(\context)$ where $(A,B)$ is mapped to the formal concept $(B^{I},B)$ is a
  $\bigvee$-preserving order-embedding of $\BV(\mathbb{T})$ in
  $\BV(\context)$. Dually, the map $\BV( \mathbb{S})\rightarrow
  \BV(\context)$ with $(A,B)\mapsto (A,A^{I})$ is a {$\bigwedge$-preserving} order
  embedding of $\BV(\mathbb{S})$ in $\BV(\context)$.
\end{proposition}

For $\context$ we observe that~\cref{prop:embed} is not applicable since a
\pkcore has potentially a modified set of objects and attributes with respect to
$\context$. Nonetheless, we can still exploit~\cref{prop:embed} in the following
way. First, there exists an order-embedding of $\BV(H,M,I\cap H\times M)$ into
$\BV (\context)$. Secondly, there is an order-embedding from $\BV(\Scon)$ into
$\BV(H,M,I\cap H\times M)$. Hence, it is easy to see that the composition of the
two maps results in an order-embedding from $\BV(\Scon)$ into $\BV(\context)$.
However, suprema and infima are not necessarily preserved. Nonetheless, the
existence of the order-embedding does in particular imply that a significant
amount of structural (conceptual) information is preserved by the \pkcore with
respect to the lattice $\BV(\context)$ and $p,q\in\N$.

In the following we want to investigate more thoroughly how concepts change when
objects/attributes are deleted or added. We start with recalling a fact
from~\cite[\pno~99]{fca-book} which is related to~\cite[Proposition~30 on
\pno~98]{fca-book}. It describes how attribute closures alter when attributes
are removed.

\begin{proposition}[Deleting Attributes]
\label{prop:1}
Let $\context = (G,M,I)$ and $\mathbb{S}=(G,N,J)$ be formal contexts with
$\mathbb{S}\leq \context$. Then,
\begin{enumerate}[i)]
\item $\forall D \in\Int(\context): (D\cap N)\in\Int(\mathbb{S})$
\item $\forall D \in \Int(\mathbb{S})\exists B\in\Int(\context): D^{J} = (B\cap N)^{I}$.
\end{enumerate}

\end{proposition}
\begin{proof}
  \begin{inparaenum}[i)]
  \item We refer the reader to~\cite[\pno\,99]{fca-book}
  \item We know that $D^{J}\in\Ext(\Scon)$ and via~\cite[Proposition~30]{fca-book} it
    follows that $D^J\in \Ext(\context)$ . Therefore, we know that $D^{JI}\in
    \Int(\context)$. With $\Scon \leq \context$ we can infer that $D=D^{JJ}\subseteq
    D^{JI}$ and $D^{JI}\cap N = D$. Hence, with $D^{JI}$ there exists a $B$ as
    required in ii).
  \end{inparaenum}
\end{proof}

\begin{corollary}[Adding Attributes]
\label{prop:2}
Let $\context = (G,M,I)$ and $\mathbb{S}=(G,N,J)$ be formal contexts where 
$\mathbb{S}\leq\context$ is true. Then,
\begin{enumerate}[i)]
\item $\forall D\in\Int(\Scon)\exists B \in \Int(\context): B\cap N = D$.
\item $\forall D\in\Int(\Scon)\setminus \Int(\context)\exists
  B\in\Int(\context)\setminus\Int(\Scon): B\cap N=D $
\end{enumerate}

\end{corollary}
\begin{proof}
  \begin{inparaenum}[i)]
    \item Use construction of $B$ from~\cref{prop:1}, part ii).
    \item Assume there is no $B$ in $\Int(\context)\setminus\Int(\Scon)$ with
      $B\cap N=D$. From i) we can then draw that $B\in
      \Int(\Scon)\cap\Int(\context)$. With $B=B\cap N=D$ this yields the
      contradiction $D\in\Int(\context)$.
  \end{inparaenum}
\end{proof}

Based on the insights so far we may draw a lemma that will drive our to be
proposed \pkcorenw-algorithm. It will employ an identity: For
$\context=(G,M,I)$ and $\Scon=(G,N,J)$ is $\Int(\context) = (\Int(\Scon) \cup
\Int(\context)\setminus\Int(\Scon)) \setminus
(\Int(\Scon)\setminus\Int(\context))$.

\begin{lemma}
  \label{lem:algo}
  Let $\context=(G,M,I)$ and $\Scon=(G,N,J)$ be formal contexts with
  $\Scon\leq\context$. Given $\Int(\Scon)$ we can compute
  $\Int(\context)$ in output polynomial time in size of
  $\Int(\context)\setminus\Int(\Scon)\cup \Int(\Scon)\setminus\Int(\context)$.
\end{lemma}
\begin{proof}
  We use the well-known \texttt{next\_closure} algorithm
  from~\cite{ganter2010basic}. We choose some order $\leq_{M}$ on $M$ such that
  $\forall m\in M\setminus N\forall n\in N: m\leq_{M} n$. We start the algorithm
  with $N$, which is the largest closure in $\Int(\Scon)$. The set
  $\Int(\context)\setminus\Int(\Scon)$ can be computed output polynomial by
  \texttt{next\_closure}, since for every element $X$ of the output we have
  $X\cap (M\setminus N)\neq\emptyset$. From~\cref{prop:2} we know that for every
  $Y\in\Int(\Scon)\setminus\Int(\context)$ there is a
  $Z\in\Int(\context)\setminus\Int(\Scon)$ with $Y\cap N=Z$. We construct the
  set $\Int(\Scon)\setminus\Int(\context)$ by $\{X\cap N\mid
  X\in\Int(\context)\setminus\Int(\Scon)\wedge(X\cap N)^{II}\neq (X\cap
  N)\}$. From~\cref{prop:2} we find that this construction yields at least all
  closures in $\Int(\Scon)\setminus\Int(\context)$ and from~\cref{prop:1} we
  know that the construction is limited to elements of $\Int(\Scon)$, again
  limited by the predicate in the construction to only those from
  $\Int(\Scon)\setminus\Int(\context)$. Altogether, we have output polynomial
  cost for $\Int(\context)\setminus\Int(\Scon)$ and one additional polynomial
  time check for every element of this set.
\end{proof}
Another identity that is useful in the experimental section is
$(\Int(\context)\cup (\Int(\Scon)\setminus\Int(\context))\setminus
(\Int(\context)\setminus\Int(\Scon))=\Int(\Scon)$. Using this the proof
from~\cref{lem:algo} can also be used to show that computing $\Int(\Scon)$
given $\Int(\context)$ is possible in output polynomial time in the
size of $\Int(\context)\setminus\Int(\Scon)\cup
\Int(\Scon)\setminus\Int(\context)$.
Since we want to explain the relation of \pkcores lattices to the concept
lattice of the original lattice we may state how we derive also the extents.

\begin{corollary}
  \label{cor:extents}
  Let $\context=(G,M,I)$ and $\Scon=(G,N,J)$ be formal contexts with
  $\Scon\leq\context$. Given $\mathfrak{B}(\Scon)$ we can compute
  $\mathfrak{B}(\context)$ in output polynomial time in size of
  $\Int(\context)\setminus\Int(\Scon)\cup \Int(\Scon)\setminus\Int(\context)$.
\end{corollary}
The only task one has to do for this is to additionally compute $X^{I}$ for
$X\in \Int(\context)\setminus \Int(\Scon)$, since all the extents from intents
ind $\Int(\Scon)\cap\Int(\context)$ remain unchanged. We also see that all
results in this section about attribute operations can be translated to object
operations through duality.

After the theoretical consideration on the impact of adding/removing
attributes to formal contexts we now want to look into the dependence of
\pkcores to removing objects.

\begin{proposition}[Object Cores]
\label{prop:3}
For two formal contexts $\context$ and $\Scon$ with
$\Scon\pkleq{p,0}\context$ and $\mathcal{F}\coloneqq\{B \in \Int(\context) \mid
\abs{B}\geq p\}$ the equality

\[\{\bigcap\mathcal{X}\mid \mathcal{X}\subseteq\mathcal{F}\} = \Int(\Scon)\quad\text{holds.} \]
\end{proposition}
\begin{proof}
  \begin{inparaitem}
  \item[$\subseteq$:] Since $\Scon$ is $p,0$-core of $\context$ we have that
    $\forall B\subseteq M: \abs{B}\geq p \Rightarrow B^{II}=B^{JJ}$. Hence,
    $\forall X\in\mathcal{F}: X^{II}=X^{JJ}\in\Int(\Scon)$. Since $\Int(\Scon)$
    is closed under intersection~\cite{fca-book} we find that for all $\mathcal{X}\subseteq\mathcal{F}:
    \bigcap\mathcal{X}\in \Int(\Scon)$.
  \item[$\supseteq$:] Assume $\exists B\in \Int(\Scon)$ with
    $B\neq\bigcap\mathcal{X}$ for all $\mathcal{X}\subseteq\mathcal{F}$. By
    definition of $\mathcal{F}$ we know that $|B|<p$. Without loss of generality
    $B$ is \emph{meet-irreducible}, i.e., there is not $\mathcal{Y}\subseteq
    \Int(\context):\bigcap \mathcal{Y}=B$. In the case where $B$ is not
    meet-irreducible there is a representation of $B$ by
    $\mathcal{Y}\subseteq\Int(\context)$, in which every element is a proper
    super set of $B$. In this set we either find a meet irreducible set or we go
    to the next representation until we have sets of cardinality $p$.  Thus,
    there exists an object $g$ of the formal context $\Scon=(H,N,J)$ with
    ${g}^{J}=B$. This contradicts $|\{g\}^{J}|\geq p$.
  \end{inparaitem}
\end{proof}

This proof employs the notion of meet irreducible intents. Computing those is
computationally challenging, in particular for larger concept lattices. We
presume that one does often consider multiple \pkcores for investigation. In
this case one may resort to the following idea: given a set of \pkcoresnw,
find a common super core, i.e., some $\Tcon$ of their original context, such that
all considered cores are $\leq \Tcon$. Compute the cover relation of the
conceptual order in $\underline{\mathfrak{B}}(\Tcon)$. Using this relation one can infer
the meet irreducible elements of $\mathfrak{B}(\Tcon)$, which are also the
possible meet irreducible elements in the concept lattices for all sub-core
contexts (or induced sub-contexts).

Based on the above we can now draw some conclusions about computing the \pkcore
concept lattice for some formal context $\context$ and $p,q\in\mathbb{N}$. But
first we may note the following.

\begin{remarks}
  \label{corestepvise}
  For $\Scon=(H,N,J)$ and $\context=(G,M,I)$ with $\Scon\pkleq{p,q}\context$ it
  holds that $\Scon\pkleq{p,0} (G,N,I\cap G\times N)$.
\end{remarks}

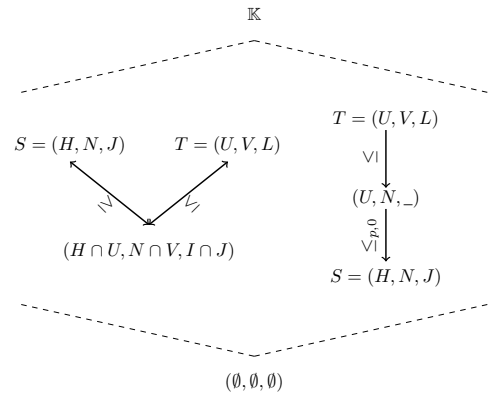
\begin{figure}[b]
  \centering
  \scalebox{.7}{
  \begin{tikzpicture}
    \title{Lattice of all induced
      sub-contexts of $\context$}

    \draw[<->,thick] (-3.5,-2.3) -- (-2,-3.5);
    \draw[<->,thick] (-2,-3.5) -- (-0.5,-2.3);
    \node[draw opacity = 0,draw=white, text=black,rotate=-45] at (-2.8,-3.1) {$\geq$};
    \node[draw opacity = 0,draw=white, text=black,rotate=45] at (-1.2,-3.1) {$\leq$};

    \draw[dashed] (0,-6) -- (-4.5,-5);
    \draw[dashed] (0,-6) -- (4.5,-5);

    \draw[dashed] (0,0) -- (-4.5,-1);
    \draw[dashed] (0,0) -- (4.5,-1);

    \node[draw opacity = 0,draw=white, text=black] at (0,0.5) {$\context$};
    \node[draw opacity = 0,draw=white, text=black] at (-2,-4) {$(H\cap
      U,N\cap V,I\cap J)$};
    \node[draw opacity = 0,draw=white, text=black] at (-0.5,-2) {$T = (U,V,L)$};

    \node[draw opacity = 0,draw=white, text=black] at (-3.5,-2) {$S = (H,N,J)$};

    \node[draw opacity = 0,draw=white, text=black] at (0,-6.5) {$(\emptyset,\emptyset,\emptyset)$};

    \draw[->,thick] (2.5,-3.2) -- (2.5,-4.2);
    \draw[->,thick] (2.5,-1.7) -- (2.5,-2.8);

    \node[draw opacity = 0,draw=white, text=black] at (2.5,-4.5) {$S = (H,N,J)$};
    \node[draw opacity = 0,draw=white, text=black] at (2.5,-3) {$(U,N,\_)$};
    \node[draw opacity = 0,draw=white, text=black] at (2.5,-1.5) {$T = (U,V,L)$};    

    \node[draw opacity = 0,draw=white, text=black, rotate=90] at (2.2,-2.2) {$\leq$};
        
    \node[draw opacity = 0,draw=white, text=black, rotate=90] at (2.2,-3.7) {$\pkleq{p,0}$};

  \end{tikzpicture}}
  \caption{Principle approach for analyzing multiple \pkcores from a formal
    context $\context$ (left) and their order/lattice relation (right).}
  \label{transform_int}
\end{figure}

Taking all the results above together we find a useful correspondence between
the concept lattices of a context, its induced sub-contexts and, in particular,
its cores. Starting with a \pkcore $\Scon\pkleq{p,q} \context$ we are able to
indicate stable concepts (with respect to $\context$ or a more general core) in
the concept lattice of the \pkcorenw. Notably, using~\cref{lem:algo} we are able
to compute efficiently the difference of the concept lattices of
$\Scon\pkleq{p,q}\Tcon\pkleq{\hat p,\hat q}\context$ with $p\leq \hat p$ and
$q\leq \hat q$.

In the last part of this section we may further generalize the findings
above. For some formal context $\context$ consider an arbitrary set of induced
sub-contexts $\mathcal{K}$. We may compare their concept lattices efficiently
using~\cref{lem:algo}, following their super/sub-context relation, as
depicted~\cref{transform_int}.

Given a formal $\context$, the set $\mathcal{K}\coloneqq\{\Scon\leq \context\}$
constitutes a complete lattice. One can see this using the map
$\mathcal{P}(G)\times\mathcal{P}(M)\to \mathcal{K}$, $(H,N)\mapsto (H,N,I\cap
(H\times N))$, which is an order isomorphism from the lattice
$\mathcal{P}(G)\times\mathcal{P}(M)$ to $\mathcal{K}$. Hence, for two arbitrary
induced sub-context $\Scon = (H,N,J)$ and $\Tcon=(U,V,L)$ of $\context=(G,M,I)$
on may compute $\mathfrak{B}(\bigvee\{\Scon,\Tcon\})$ and
$\mathfrak{B}(\bigwedge\{\Scon,\Tcon\})$ in order to infer $\mathfrak{B}(\Tcon)$
efficiently using $\mathfrak{B}(\Scon)$, or vice versa.  The set of all \pkcores
is contained in $\mathcal{K}$, however, it does not constitute a lattice. To see
this a counter example is presented in~\cref{counter}.

\subsection{A Small Case Study}
\label{sec:smallcase}
We apply our notion for \pkcores on a particularly small example, the
\emph{Forum Romanum} (FR) context (\cite[Figure 1.16]{fca-book}), in order to
study the applicability to real world data sets. The data set consists of
monuments on the Forum Romanum (objects) and their star ratings by different
travel guides (attributes). In~\cref{fig:lattice-reduce} we depicted the concept
lattice for FR and indicated by the red dashed lines the $2,4$-core of
FR.

At least all concept between the red lines remain after the core
reduction. In detail, the parameter $p=2$ results in removing all objects that
have a derivation of size two or less, as indicated by the upper horizontal
dashed line. We understand (acc. to~\cref{prop:3}) that in this process all
\mbox{\emph{join-irreducible}} concepts $(A,B)$, i.e., $\nexists\mathcal{F}\subseteq
\mathfrak{B}(FR):\bigvee \mathcal{F}{=}(A,B)$, above the $p=2$ threshold are
removed. For example, the concepts above the horizontal red dashed line having
the short hand notation labels \emph{B*, GB*}, and \emph{P*}, are
join-irreducible and therefore removed. Their attributes are then contained by
those lower concepts that are in cover relation to the removed concepts.  In
contrast, the concept with short hand label \emph{M*} is join-reducible and is
therefor closed after the removal of objects. The removal of attributes results
in dually observations, i.e., \emph{meet-irreducible} concepts are removed.

\section{Implications of $pq$-Cores}
In this section we study the relation of implicational theories of \pkcores with
respect to the original context.  We start with investigating the impact of
object set manipulations. We consider in the following two formal contexts
$\context = (G,M,I)$ and $\Scon = (H,M,J)$ with $\Scon\leq\context$. By removing
objects we possibly remove unique counter examples $g\in G$ to some invalid
implication $A\to B$, i.e., $B’\not\subseteq A’$ but $B’\setminus\{g\}\subseteq
A’$. Hence, new valid implications can emerge in $\Scon$. On the other hand,
valid implications in $\context$ cannot be disproved by removing objects. Thus,
$\Th(\context) \subseteq \Th(\Scon)$. Cores with $p\in\N$ and $q=0$ are of
particular interest to us due to~\cref{corestepvise}. For those, i.e.,
$\Scon\pkleq{p,0}\context$, we find that all valid implications $A \to B$ in
$\Th(\Scon)\setminus\Th(\context)$ have $\abs{A} < p$, since in this core we
only remove objects $g\in G$ with $|\{g\}’|<p$, which are only able to refute
implications with premise $|A|<p$. For the special case of
$\Scon\pkleq{0,q}\context$ we can deduce that
$\Th(\Scon)\subseteq\Th(\context)$.

There are two essential notions when discussing implications in data sets,
\emph{confidence} and \emph{support}. The support of an implication $A\to
B\in\Th(\context)$ is defined by $\supp_{\context}(A\to B)\coloneqq
\nicefrac{|(A\cup B)’|}{|G|}$ and the confidence by $\conf_{\context}(A\to
B)\coloneqq \nicefrac{|(A\cup B)’|}{|A’|}$. We may note that only implications
with confidence one are considered valid in FCA and therefore included in
$\Th(\context)$. Nonetheless there is a strong correspondence to the realm of association rules.

\begin{proposition}[Core Implications]\label{Core-propositions}
  Let $\context ,\Scon$ be formal contexts, with $\Scon\pkleq{p,q} \context$
  where $\context=(G,M,I)$ and $\Scon=(H,N,J)$. For all $A\to B\in\Th(\Scon)$ is
  \begin{enumerate}[i)]
  \item $\nicefrac{\abs{H}}{\abs{G}}\cdot\sup_{\Scon}(A\rightarrow B) \leq
    \sup_{\context}(A\rightarrow B)$
  \item $\sup_{\context}(A\rightarrow B) \leq
    \nicefrac{\abs{H}}{\abs{G}}\cdot \sup_{\Scon}(A\rightarrow B) 
    +\nicefrac{|G\setminus H|}{|G|}$
  \item $\conf_{\context}(A\to B)\geq \nicefrac{|(A\cup B)^{J}|}{|A^{J}|+
      |G\setminus H|}$
  \item $\abs{A}\geq p \implies \conf_{\context}(A\to B) = 1$

\item $ \abs{A\cup B}\geq p \implies \sup_{\context}(A\,{\rightarrow}\,B) =
  \nicefrac{\abs{H}}{\abs{G}} \cdot \sup_{\Scon}(A\,{\rightarrow}\,B)$.
  \end{enumerate}
\end{proposition}
\begin{proof}
  \begin{inparaenum}[i)]
  \item Since $J\subseteq I$ we can infer that $|A^{J}|\leq |A^{I}|$, we can infer
    $\nicefrac{|H|}{|G|}\cdot \supp_{\Scon}(A)\leq \supp_{\context}(A)$.
  \item With the same argument as in i) we can infer $|A^{I}|\leq |A^{J}|+
    |G\setminus H|$, from which one can deduce the statement. 
  \item Using i) and ii), which would be the best case / worst case for
    supports, since all additional objects are counter examples for $A\to B$, we
    find $\conf_{\context}(A\to B)=\nicefrac{|(A\cup B)^{I}}{|A^{I}|}$ is equal
    to $\nicefrac{H}{G}\cdot\supp_{\Scon}(A\to B)$ divided by
    $\nicefrac{H}{G}\cdot\supp_{\Scon}(A)+\nicefrac{|G\setminus H|}{|G|}$. This
    can be simplified to $\conf_{\context}(A\to B)=\nicefrac{|(A\cup
      B)^{J}|}{|A^{J}|+|G\setminus H|}$. 
  \item For $|A|\geq p$ we have $A^{I}=A^{J}$ by definition of \pkcores and also
    $(A\cup B)^{J}=(A\cup B)^{I}$. Together with the definition of confidence we
    obtain the statement.
  \item From $\abs{A\cup B}\geq p$ we find that $|(A\cup B)^{I}|=|(A\cup
    B)^{J}|$, which results in a equality in i).
  \end{inparaenum}
\end{proof}

Note that i), ii), and iii) are also valid for sub-contexts. We now study minimal
representations of implicational theories, i.e., the canonical base of
$\Th(\context)$ for some formal context $\context$.

\label{sec:canbase}
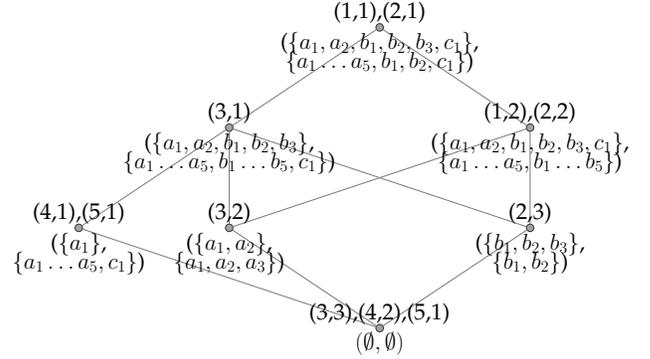
\begin{figure}[t]
  \centering
{\footnotesize\begin{cxt}
    \cxtName{}
    \att{$a_1$}
    \att{$a_2$}
    \att{$a_3$}
    \att{$a_4$}
    \att{$a_5$}
    \att{$b_1$}
    \att{$b_2$}
    \att{$b_3$}
    \att{$b_4$}
    \att{$b_5$}
    \att{$c_1$}
    \obj{xxxxx.....x}{ $a_1$}
    \obj{xxx........}{ $a_2$}
    \obj{.....xxx...}{ $b_1$}
    \obj{.....xx.x..}{$b_2$}
    \obj{.....xx..x.}{$b_3$}
    \obj{...xx......}{$c_1$}
  \end{cxt}}\ \vspace{0.2cm}

\scalebox{0.5}{\colorlet{mivertexcolor}{black!80}
\colorlet{jivertexcolor}{black!80}
\colorlet{vertexcolor}{black!80}
\colorlet{bordercolor}{black!80}
\colorlet{linecolor}{gray}
\tikzset{vertexbase/.style={semithick, shape=circle, inner sep=2pt, outer sep=0pt, draw=bordercolor},%
  vertex/.style={vertexbase, fill=vertexcolor!45},%
  mivertex/.style={vertexbase, fill=mivertexcolor!45},%
  jivertex/.style={vertexbase, fill=jivertexcolor!45},%
  divertex/.style={vertexbase, top color=mivertexcolor!45, bottom color=jivertexcolor!45},%
  conn/.style={-, thick, color=linecolor}%
}
\tikzstyle{line2} = [label distance=0.5cm]
\begin{tikzpicture}[scale=8, font=\LARGE]
  \begin{scope} 
    \begin{scope} 
      \foreach \nodename/\nodetype/\xpos/\ypos in {%
        0/vertex/0.5/0.0,
        1/divertex/0.0/0.333333333333336,
        2/divertex/1.0/0.333333333333336,
        3/divertex/0.0/0.66666666666667,
        4/divertex/1.0/0.66666666666667,
        5/vertex/0.5/1.0,
        6/divertex/-0.5/0.333333333333336
      } \node[\nodetype] (\nodename) at (\xpos, \ypos) {};
    \end{scope}
    \begin{scope} 
      \path (1) edge[conn] (4);
      \path (6) edge[conn] (3);
      \path (6) edge[conn] (0);
      \path (2) edge[conn] (3);
      \path (2) edge[conn] (4);   
      \path (3) edge[conn] (5);
      \path (1) edge[conn] (3);
      \path (0) edge[conn] (1);
      \path (0) edge[conn] (2);
      \path (4) edge[conn] (5);
    \end{scope}
    \begin{scope} 
      \foreach \nodename/\labelpos/\labelopts/\labelcontent in {%
        0/below//{$(\emptyset,\emptyset)$},
        0/above//{(3,3),(4,2),(5,1)},
        1/below//{($\{a_1, a_2\}$,},
        1/below/line2/{$\{a_1, a_2, a_3\}$)},
        1/above//{(3,2)},
        2/below//{($\{b_1, b_2, b_3\}$,},
        2/below/line2/{$\{b_1, b_2\}$)},
        2/above//{(2,3)},
        3/below//{($\{a_1, a_2, b_1, b_2, b_3\}$,},
        3/below/line2/{$\{a_1 \dots a_5, b_1 \dots b_5, c_1\}$)},
        3/above//{(3,1)},
        4/below//{($\{a_1, a_2, b_1, b_2, b_3, c_1\}$,},
        4/below/line2/{$\{a_1 \dots a_5, b_1 \dots b_5\}$)},
        4/above//{(1,2),(2,2)},
        5/below//{($\{a_1, a_2, b_1, b_2, b_3, c_1\}$,},
        5/below/line2/{$\{a_1 \dots a_5, b_1, b_2, c_1\}$)},
        5/above//{(1,1),(2,1)},
        6/below//{($\{a_1\}$,},
        6/below/line2/{$\{a_1 \dots a_5, c_1\}$)},
        6/above//{(4,1),(5,1)}
      } \coordinate[label={[\labelopts]\labelpos:{\labelcontent}}](c) at (\nodename);
    \end{scope}
  \end{scope}
\end{tikzpicture}}
\caption{An example context (upper) and the order relation of all \pkcores
  (lower). Each node in the order diagram represents a \pkcore with its $p,q$
  values written above the node.}
  \label{counter}
\end{figure}

\begin{figure}
  \centering
  \input{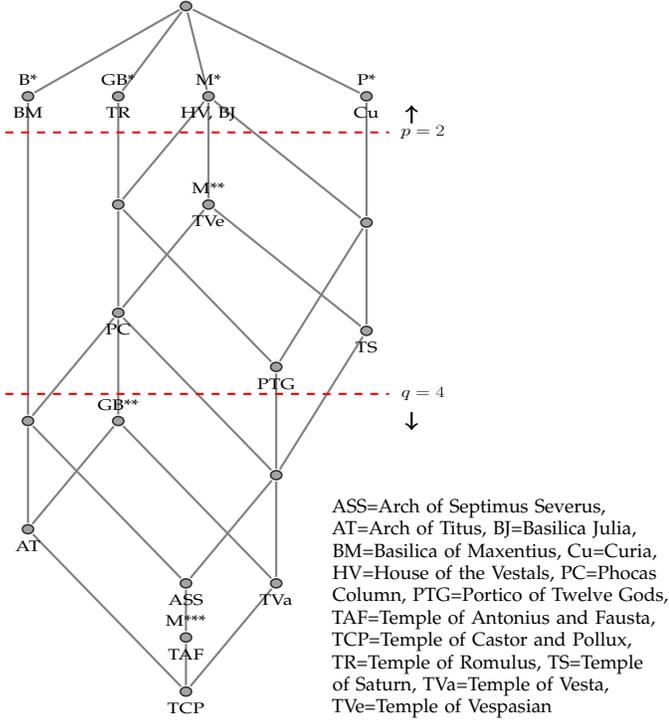}
  \caption[]{The $2,4$-core of the concept lattice is indicated by the red
    lines. All objects present in the short hand notation above the $p=2$
    barrier are removed as well as all attributes below $p=4$ line are removed.}
  \label{fig:lattice-reduce}
\end{figure}

The next logical step would be to partially derive the canonical base for some
formal context $\context$ using a \pkcorenw. However, this endeavor is so far
not understood. In the simple case of formal contexts on disjoints attribute
sets, i.e., computing the canonical base of $(G,N_1\dot\cup N_2,J_2\dot\cup
J_2)$ using the bases of $(G,N_1,J_1),(G,N_2,I_2)$, we refer the reader
to~\cite{mergebase}. Nonetheless, we may yield some results for the \emph{canonical
  direct bases}~\cite{ganter2001implikationen, fca-book} (CDB), i.e., a complete, sound and iteration free basis. Such a
basis for a formal context $\context=(G,M,I)$ is constituted by set of
\emph{proper premises}, i.e., sets $A\subseteq M$ where $A^{II}\setminus (A\cup
\bigcup_{B\subsetneq A} B^{II})\neq \emptyset$ does hold, cf.~\cite{directBase}.

\begin{proposition}[Induced Contexts CDB]
  Let $\context = (G,M,I)$, $\Scon = (G,N,J)$ be two formal contexts
  with $\Scon \leq \context$ and let $\mathcal{L}_p(\Scon),\mathcal{L}_p(\context)$ be
  their canonical direct bases, then
\[\mathcal{L}_p(\Scon)\subseteq \mathcal{L}_p(\context).\]
\end{proposition}
\begin{proof}
  Let $A\subseteq N$ be a proper premise of $\Scon$. Hence, we know that
  $A^{JJ}\setminus (A \cup \bigcup_{B\subset A} B^{JJ})
  \neq\emptyset$. Following, there is an $n\in N$ with $n\notin A$ and $n\notin
  B^{JJ}$ for all $B\subsetneq A$. With the following~\cref{lem:prop2}, we find
  that forall $B\subset A$ we have $B^{II}=B^{JJ} \cup (B^{II}\setminus
  N)$. Therefore, we find that $n\notin B^{II}$. From this we can conclude that
  $n\in A^{II}\setminus (A \cup \bigcup_{B\subset A} B^{II})$ which is therefore
  not empty.
\end{proof}
\begin{lemma}
  \label{lem:prop2}
  Let $\context = (G,M,I), \Scon = (G,N,J)$ be two formal contexts with $\Scon
  \leq \context$ and $B\subseteq N$, then $B^{II}=B^{JJ}\cup (B^{II} \setminus
  N)$.
\end{lemma}
\begin{proof}
  \begin{inparaitem}
  \item[$\subseteq$:] The only interesting case is $n\in B^{II}$ with $n\notin
    B$. Assume $n\notin B^{JJ}\cup B^{II}\setminus N$ which yields $n\in
    N\setminus B^{JJ}$. This demands the existence of ${g\in G:(g,n)\in
    I\wedge (g,n)\notin J}$. Since $\Scon\leq \context$ on the same object set
    $G$ this results in a contradiction.
  \item[$\supseteq$:] Since $J=I\cap (G\times N)$ we know that $B^{J}=B^{I}$
    which results in $B^{JJ}=B^{IJ}\subseteq B^{II}$.
  \end{inparaitem}
\end{proof}

\section{Experimental Study}\label{coreChar}
We collected different theoretical properties of \pkcoresnw. In this section we
want to study experimentally their applicability on real-world data sets. The
most pressing question is to identify particularly interesting cores of a given
formal context. A commonly used technique to assess the interestingness of k-cores
in networks is to investigate the number of connected components depending on
the core parameter $k$. A well-known observation is that the number of connected
components increases the greater $k$ is. Parameters that are considered
interesting are those around the steepest rate of increase in the number of
components. Also often considered are changes of some valuation function, such as
the size of the largest connected component or some network statistical
property. We will adapt the former idea and analyze the component structures.

\subsubsection*{Data Sets}
We conduct our investigation on five various sized data sets from different
domains.
\begin{inparadesc}
\item[Living beings in Water] is the well known FCA data
  set~\cite[Figure1.1]{fca-book}. It consists of living beings as objects and
  their properties as attributes.
\item[Forum Romanum] as already used in~\cref{sec:smallcase}, is also taken
  from \cite{fca-book}. It is made of places of interest as objects and their ratings in
  different tour guides as attributes.
\item[Spices] is created by the authors. The objects are dishes and the
  attributes are spices to be used for these dishes. The incidence
  relation is extracted from a spices planer~\cite{herbs}.
\item[Mushroom] is an often used classification data set provided by
  UCI~\cite{Dua:2019}. The objects are mushrooms and the non-binary attributes
  are common mushroom properties. Those were scaled using a nominal scale.
\item[The Pocket Knives] data set was self-created by the authors through
  crawling the Victorinox AG website\footnote{\url{https://www.victorinox.com}}
  in April 2019. The context contains all pocket knives as objects and their
  features as attributes.
\item[Wiki44k] was created in an experimental study~\cite{hanika2019discovering}
  on finding implications in Wikidata. It is a scaled context drawn from the
  most dense part of the Wikidata knowledge graph.
\end{inparadesc}
All presented data sets are available in the FCA software
\texttt{conexp-clj}~\cite{conf/icfca/HanikaH19} through 
GitHub.\footnote{\url{https://github.com/tomhanika/conexp-clj}} We
collected their numerical properties in~\cref{table:datasets}.

\begin{table}[t]
  \centering
  \caption{Numerical description of data sets. We included the number of
    non-empty \pkcores as well as the number of formal concepts.}
  \begin{tabular}{l|r|r|r|r|r}
    \toprule
    Name&$|G|$&$|M|$&$|\mathfrak{B}(\context)$&\# \pkcores&
                                                            density\\
    \midrule
    Water & 8&9&19&20&0.47\\
    Romanum & 14&7&19&34&0.45\\
    Spices & 56&37&421&136&0.23\\
    Knives & 159&108&1061&1072&0.11\\
    Mushroom & 8124&119&238710&80136&0.22\\
    Wiki44k & 45021&101&21923&$\approx$ 98000&0.05\\
    \bottomrule
  \end{tabular}
  \label{table:datasets}
\end{table}

\subsubsection*{Interesting \pkcores}

For all data sets we applied different combinations of parameters $p$ and $q$
and evaluated to what extent this leads to interesting \pkcores using the
steepest increase method. For this we regarded all non-empty \pkcores as
bipartite graph and counted the resulting connected components. We observed that
no data set has a \pkcore with more than one connected component. This is
surprising since constructing a formal context falling apart into multiple
connected components for some $p$ and $q$ is easy. This might indicate that
real-world data sets do not exhibiting this property. However, we acknowledge
that the number of considered data sets is comparatively low. Nonetheless, this
observation might be attributed to the following fact: in all data sets there is
a small number of objects with high support, i.e., many attributes, covering in
union all attributes and having at least pairwise one attribute in common.
These objects are contained in all \pkcoresnw. Hence, we need to adapt the idea
of components to the realm of formal contexts differently. For this we consider
the context size distribution among all \pkcoresnw. In these distribution we may
characterize sub-contexts that are removed while computing a \pkcore as
structural components. This is in contrast to the classical component analysis
for graph $k$-cores. Using those we define interesting \pkcores as those where a
further increase of $p$ or $q$ would result in a high increase in the size of
the removed structural component. In our experiments we find that there are many such
critical $p$ and $q$ for the investigated data sets. To narrow this set we
propose the following pragmatic selection criteria due to computational
limitations:
\begin{inparaenum}
\item The size of a selected core should be in the range of computational
  feasibility (with respect to the to be employed analysis procedures).
\item The parameters $p$ and $q$ of a selected core should differ in
  magnitudes, i.e., either $p \ll q$ or $p\gg q$. 
\end{inparaenum}
The interpretation of either criterion depends on the particular data analysis
application. For example, if one is more interested in keeping a larger
attribute domain then one should choose an interesting core with low $q$ and
high $p$. Analogously one might want to keep more objects.

This being said we want to propose a different approach for characterizing
interesting \pkcoresnw. In contrast to solely considering a \pkcore
$\Scon\pkleq{p,q}\context$ of some context $\context$ one might look into the
concept lattice that is created by this \pkcorenw, i.e., $\mathfrak{B}(\Scon)$.
With this approach the size of the resulting concept lattice could be a
criterion to select a \pkcorenw. The motivation for this is that we rather
select a \pkcores depending on the entailed conceptual knowledge than purely on
contextual size. This approach is computational costly since we need to compute
a large number of concept lattices. However, relying on~\cref{lem:algo},
\cref{prop:3} and~\cref{corestepvise} we may ease this cost significantly.
Analogously we propose selection criteria:
\begin{inparaenum}
\item The diagram of a selected core lattice should be human readable,
  (e.g.,
  the number of concepts should be in a human feasible range)
\item The parameters $p$ and $q$ of a selected core lattice should differ in
  magnitudes, i.e., either $p \ll q$ or $p\gg q$.
\end{inparaenum}
Again, the concrete employment of either criterion depends on the particular
data analysis application. For example, we find a lattice with more than thirty
concepts too large for human comprehension, even if drawn with sophisticated
drawing algorithms. Hence, we will consider this number for the rest of this
work as bound. On a final note in this section, we consider the special cases of
object- and attribute cores not to be interesting. They remove attributes or
objects simply by their object/attribute support and do not represent an
interesting sub-structure.

\subsubsection*{Experiment: Water}
We analyze the \emph{living beings and water} context~\cref{fig:example-core}
and present our core analysis in~\cref{water-core-lattices}. For this we
computed the size of all core concept lattices. A first observation is that
interesting cores, with respect to our just introduced notion of
interestingness, are the $4,3$-, $3,4$- and $2,4$-core. We suspect that they
include important knowledge. Increasing the core parameters more would lead to
an (almost) empty concept lattice. From this list of interesting \pkcores we
present the lattice diagram of $4,3$ in~\cref{water-core-lattices}. This lattice
contains thirteen formal concepts in contrast to the nineteen in the original
concept lattice. The $4,3$-core captures a significant portion of knowledge from
the original domain, however, only six out of eight objects and seven out of
nine attributes are in the picture. We can still infer two different groups of
beings, plants and animals. Nonetheless, the original lattice is much more
refined. For example, the original concept lattice is more distinct in the
subsets of beings that need \emph{chlorophyll} or those who \emph{can move}
around. We consider the \pkcore to be a more coarse representation of the
entailed domain knowledge.

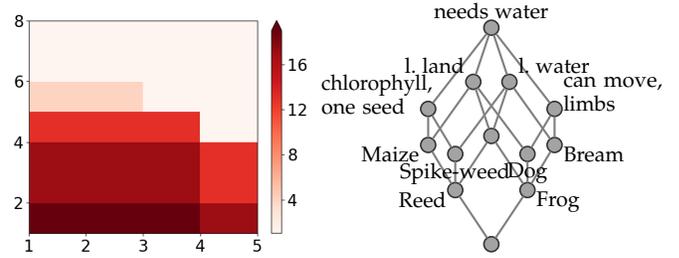
\begin{figure}
  \label{water-core-lattices}
  \centering
  \begin{minipage}{0.44\linewidth}
    \includegraphics[scale=0.3,trim=8ex 1ex 3ex 0ex, clip]{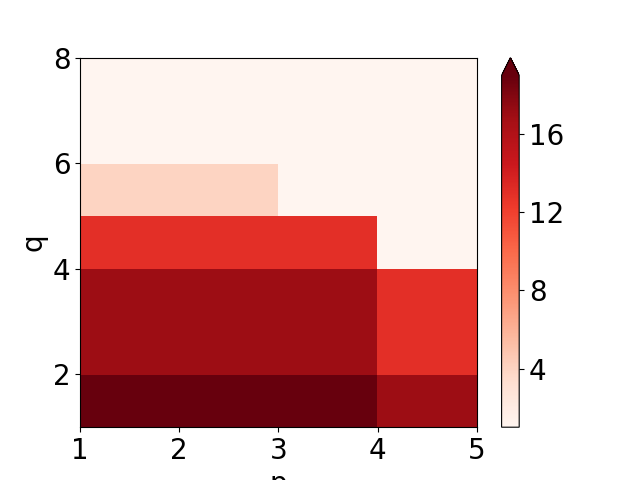}
  \end{minipage}
  \begin{minipage}{0.49\linewidth}
    \colorlet{mivertexcolor}{black!80}
\colorlet{jivertexcolor}{black!80}
\colorlet{vertexcolor}{black!80}
\colorlet{bordercolor}{black!80}
\colorlet{linecolor}{gray}
\tikzset{vertexbase/.style={semithick, shape=circle, inner sep=2pt, outer sep=0pt, draw=bordercolor},%
  vertex/.style={vertexbase, fill=vertexcolor!45},%
  mivertex/.style={vertexbase, fill=mivertexcolor!45},%
  jivertex/.style={vertexbase, fill=jivertexcolor!45},%
  divertex/.style={vertexbase, top color=mivertexcolor!45, bottom color=jivertexcolor!45},%
  conn/.style={-, thick, color=linecolor}%
}
\tikzstyle{line2} = [label distance=0.5cm]
\tikzstyle{line3} = [text width=1.3cm,yshift=0.2cm]
\tikzstyle{line4} = [yshift=0.1cm]
\begin{tikzpicture}[scale=0.12,font=\footnotesize]
  \begin{scope} 
    \begin{scope} 
      \foreach \nodename/\nodetype/\xpos/\ypos in {%
        0/vertex/0/0,
        1/jivertex/-4/6,
        2/jivertex/4/6,
        3/jivertex/-4/10,
        4/jivertex/4/10,
        5/jivertex/-7/11,
        6/jivertex/7/11,
        7/vertex/0/12,
        8/mivertex/-7/15,
        9/mivertex/7/15,
        10/mivertex/-2/18,
        11/mivertex/2/18,
        12/vertex/0/24
      } \node[\nodetype] (\nodename) at (\xpos, \ypos) {};
    \end{scope}
    \begin{scope} 
      \path (3) edge[conn] (11);
      \path (11) edge[conn] (12);
      \path (5) edge[conn] (8);
      \path (1) edge[conn] (5);
      \path (1) edge[conn] (7);
      \path (0) edge[conn] (2);
      \path (2) edge[conn] (4);
      \path (7) edge[conn] (11);
      \path (7) edge[conn] (10);
      \path (3) edge[conn] (8);
      \path (6) edge[conn] (11);
      \path (4) edge[conn] (9);
      \path (9) edge[conn] (12);
      \path (10) edge[conn] (12);
      \path (8) edge[conn] (12);
      \path (0) edge[conn] (1);
      \path (2) edge[conn] (7);
      \path (6) edge[conn] (9);
      \path (5) edge[conn] (10);
      \path (2) edge[conn] (6);
      \path (1) edge[conn] (3);
      \path (4) edge[conn] (10);
    \end{scope}
    \begin{scope} 
      \foreach \nodename/\labelpos/\labelopts/\labelcontent in {%
        1/below left/line4/{Reed},
        2/below right/line4/{Frog},
        3/below//{Spike-weed},
        4/below//{Dog},
        5/below left/line4/{Maize},
        6/below right/line4/{Bream},
        8/left/line3/{\baselineskip=1pt chlorophyll, one seed},
        9/right/line3/{\baselineskip=1pt can move, limbs},
        10/above left//{l.~land},
        11/above right//{l.~water},
        12/above//{needs water}
      } \coordinate[label={[\labelopts]\labelpos:{\labelcontent}}](c) at (\nodename);
    \end{scope}
  \end{scope}
\end{tikzpicture}
  \end{minipage}
  \caption{Figure on the left shows the concept lattice sizes for all \pkcores
    of living beings and water data set, the abscissa indicates $p$ and the
    ordinate $q$. On the right we present the $4,3$-core.}
\end{figure}

\subsubsection*{Experiment: Spices}
\begin{figure}[b]\label{herb-sizes}
  \centering
  \includegraphics[width=6cm,trim=6ex 2ex 2ex 7ex, clip]{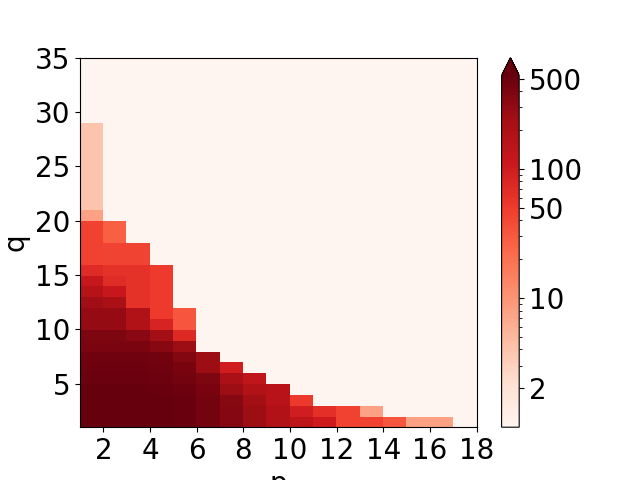}
  \caption{Concept lattice sizes for all \pkcores of spices data set, the
    abscissa indicates $p$ and the ordinate $q$.}
\end{figure}

In this experiment we analyze a spice recommendation data set. This context is
derived from a spice planer published in~\cite{herbs}. It contains 56 meals and
37 spices. Meals in the data set cover nine categories which are not part of the
formal context. There are fifteen vegetables, nine meat, three poultry, five
fish, five potato, four rice dishes, as well as three sauces, eight baked goods
and four diverse dishes. The incidence relation is which meal requires which
spices. The resulting concept lattice of the original context contains 531
formal concepts. The results of applying \pkcores to this data set with
different parameters is depicted in~\cref{herb-sizes}. There is a great number
of candidate cores to be considered, i.e., cores with a steep decrease in the
number of formal concepts while increasing parameters $p$ or $q$. However, many
of those are still very large with respect to the number of formal concepts,
e.g., $5,7$-core or the $9,4$-core. Following our pragmatic criterion for human
readability those are not interesting. In contrast is the $5,11$-core (cf. light
red color in figure). In this core lattice the parameters $p$ and $q$ are
approximately equally sized.  Hence, it only covers a dense object and attribute
selection. In particular there are twelve dishes using six spices.

As another selection we present two different cores exhibiting a large attribute
coverage and large object coverage respectively. A real-world motivation for
this is: one wants to cook lots of different dishes with possibly fewer
spices; one is focused on a diverse usage of spices with possibly fewer
meals. We exemplify this with the $2,18$-core and the $14,1$-core, as depicted
in~\cref{herbs}. The $2,18$-core includes 28 concepts with 33 out of the 56
dishes. The $14,1$-core has 32 concepts with 29 out of the 37 spices. While
having less than 10\% of the size of the original concept lattice, both concept
lattices cover a vast amount of human recognizable knowledge.  A thorough
investigation with respect to implications is done in later in this work.

\begin{figure}\label{herbs}
    \hspace{1cm}\scalebox{0.5}{\colorlet{mivertexcolor}{black!80}
\colorlet{jivertexcolor}{black!80}
\colorlet{vertexcolor}{black!80}
\colorlet{bordercolor}{black!80}
\colorlet{linecolor}{gray}
\tikzset{vertexbase/.style={semithick, shape=circle, inner sep=2pt, outer sep=0pt, draw=bordercolor},%
  vertex/.style={vertexbase, fill=vertexcolor!45},%
  mivertex/.style={vertexbase, fill=mivertexcolor!45},%
  jivertex/.style={vertexbase, fill=jivertexcolor!45},%
  divertex/.style={vertexbase, top color=mivertexcolor!45, bottom color=jivertexcolor!45},%
  conn/.style={-, thick, color=linecolor}%
}  
\tikzstyle{line2} = [label distance=0.5cm]
\tikzstyle{line3} = [label distance=1cm]
\tikzstyle{line4} = [label distance=1.5cm]
\tikzstyle{lines} = [xshift=0.5cm]
\tikzstyle{line2s} = [label distance=0.5cm,xshift=0.5cm]
\tikzstyle{line3s} = [label distance=1cm,xshift=0.5cm]
\tikzstyle{line4s} = [label distance=1.5cm,xshift=0.5cm]
\begin{tikzpicture}[scale=0.3]
  \begin{scope} 
    \begin{scope} 
      \foreach \nodename/\nodetype/\xpos/\ypos in {%
        0/vertex/0/0,
        1/jivertex/4/8,
        2/jivertex/-3/9,
        3/jivertex/-11/13,
        4/jivertex/11/13,
        5/vertex/-1/17,
        6/jivertex/-14/18,
        7/jivertex/12/20,
        8/vertex/-7/21,
        9/vertex/7/21,
        10/vertex/17/23,
        11/vertex/-16/24,
        12/vertex/0/26,
        13/vertex/10/26,
        14/vertex/-8/28,
        15/vertex/-19/29,
        16/vertex/19/29,
        17/vertex/-13/31,
        18/vertex/12/32,
        19/vertex/0/34,
        20/vertex/-6/36,
        21/vertex/6/36,
        22/mivertex/-15/37,
        23/mivertex/15/37,
        24/mivertex/-6/44,
        25/mivertex/6/44,
        26/mivertex/0/46,
        27/vertex/0/54
      } \node[\nodetype] (\nodename) at (\xpos, \ypos) {};
    \end{scope}
    \begin{scope} 
      \path (25) edge[conn] (27);
      \path (7) edge[conn] (19);
      \path (7) edge[conn] (13);
      \path (7) edge[conn] (16);
      \path (18) edge[conn] (23);
      \path (18) edge[conn] (26);
      \path (0) edge[conn] (4);
      \path (0) edge[conn] (3);
      \path (0) edge[conn] (2);
      \path (0) edge[conn] (1);
      \path (4) edge[conn] (7);
      \path (4) edge[conn] (10);
      \path (4) edge[conn] (9);
      \path (4) edge[conn] (12);
      \path (23) edge[conn] (27);
      \path (8) edge[conn] (14);
      \path (8) edge[conn] (21);
      \path (8) edge[conn] (17);
      \path (19) edge[conn] (25);
      \path (19) edge[conn] (24);
      \path (5) edge[conn] (18);
      \path (5) edge[conn] (17);
      \path (20) edge[conn] (24);
      \path (20) edge[conn] (26);
      \path (15) edge[conn] (24);
      \path (15) edge[conn] (22);
      \path (14) edge[conn] (25);
      \path (14) edge[conn] (22);
      \path (3) edge[conn] (8);
      \path (3) edge[conn] (11);
      \path (3) edge[conn] (12);
      \path (3) edge[conn] (6);
      \path (21) edge[conn] (25);
      \path (21) edge[conn] (26);
      \path (11) edge[conn] (20);
      \path (11) edge[conn] (15);
      \path (11) edge[conn] (17);
      \path (24) edge[conn] (27);
      \path (2) edge[conn] (5);
      \path (2) edge[conn] (11);
      \path (2) edge[conn] (9);
      \path (10) edge[conn] (18);
      \path (10) edge[conn] (21);
      \path (10) edge[conn] (16);
      \path (9) edge[conn] (18);
      \path (9) edge[conn] (20);
      \path (9) edge[conn] (13);
      \path (13) edge[conn] (23);
      \path (13) edge[conn] (24);
      \path (1) edge[conn] (8);
      \path (1) edge[conn] (5);
      \path (1) edge[conn] (10);
      \path (12) edge[conn] (19);
      \path (12) edge[conn] (20);
      \path (12) edge[conn] (21);
      \path (26) edge[conn] (27);
      \path (17) edge[conn] (26);
      \path (17) edge[conn] (22);
      \path (22) edge[conn] (27);
      \path (16) edge[conn] (25);
      \path (16) edge[conn] (23);
      \path (6) edge[conn] (19);
      \path (6) edge[conn] (15);
      \path (6) edge[conn] (14);
    \end{scope}
    \begin{scope} 
      \foreach \nodename/\labelpos/\labelopts/\labelcontent in {%
        0/below//{Fried Fish},
        1/below//{Pork Meat, Veal},
        1/below/line2/{Duck, Grilled Fish},
        1/below/line3/{Steamed Fish, Bright Sauce},
        2/below//{Beef, Lamb},
        3/below//{Mushrooms, Pottage, Roast Potato},
        3/below/line2/{Stew, Stove, Potato},
        3/below/line3/{Dark Sauce},
        4/below//{Dip with Herbs},
        4/below/line2/{Hash, Asian Rijsttafel},
        6/below//{Potato Casserole},
        7/below//{Cauliflower},
        10/below/lines/{Risotto, Chicken},
        10/below/line2s/{Vegetable Dauphinoise},
        10/below/line3s/{Baked Fish},
        11/below//{Pasta, Pizza},
        13/below//{Curry Rice},
        14/below//{Omelet},
        16/below//{Carrot},
        19/below//{Beans},
        20/below//{Goulash},
        21/below//{Tomato, Spinach},
        22/above//{Thyme},
        23/above//{Curry},
        24/above//{Cayenne Pepper},
        25/above//{Pepper White},
        26/above//{Garlic}
      } \coordinate[label={[\labelopts]\labelpos:{\labelcontent}}](c) at (\nodename);
    \end{scope}
  \end{scope}
\end{tikzpicture}

}
    \scalebox{0.5}{\colorlet{mivertexcolor}{black!80}
\colorlet{jivertexcolor}{black!80}
\colorlet{vertexcolor}{black!80}
\colorlet{bordercolor}{black!80}
\colorlet{linecolor}{gray}
\tikzset{vertexbase/.style={semithick, shape=circle, inner sep=2pt, outer sep=0pt, draw=bordercolor},%
  vertex/.style={vertexbase, fill=vertexcolor!45},%
  mivertex/.style={vertexbase, fill=mivertexcolor!45},%
  jivertex/.style={vertexbase, fill=jivertexcolor!45},%
  divertex/.style={vertexbase, top color=mivertexcolor!45, bottom color=jivertexcolor!45},%
  conn/.style={-, thick, color=linecolor}%
}
\tikzstyle{line2} = [label distance=0.5cm]
\tikzstyle{line3} = [label distance=1.1cm]
\tikzstyle{line4} = [label distance=1.5cm]
\tikzstyle{closer} = [label distance=-0.2cm]
\tikzstyle{lines} = [xshift=1.2cm]
\tikzstyle{line2s} = [label distance=0.5cm,xshift=0.5cm]
\tikzstyle{line3s} = [label distance=1cm,xshift=0.5cm]
\tikzstyle{line4s} = [label distance=1.5cm,xshift=0.5cm]
\begin{tikzpicture}[scale=0.3]
  \begin{scope} 
    \begin{scope} 
      \foreach \nodename/\nodetype/\xpos/\ypos in {%
        0/vertex/0/0,
        1/jivertex/0/8,
        2/jivertex/-6/10,
        3/jivertex/6/10,
        4/jivertex/-15/17,
        5/jivertex/15/17,
        6/vertex/-6/18,
        7/vertex/6/18,
        8/vertex/0/20,
        9/vertex/-15/25,
        10/vertex/15/25,
        11/vertex/-21/27,
        12/vertex/-9/27,
        13/vertex/9/27,
        14/vertex/21/27,
        15/vertex/0/28,
        16/vertex/0/34,
        17/vertex/-21/35,
        18/vertex/-9/35,
        19/vertex/9/35,
        20/vertex/21/35,
        21/vertex/-15/37,
        22/vertex/15/37,
        23/vertex/0/42,
        24/vertex/-6/44,
        25/vertex/6/44,
        26/mivertex/-15/45,
        27/mivertex/15/45,
        28/mivertex/-6/52,
        29/mivertex/6/52,
        30/mivertex/0/54,
        31/vertex/0/62
      } \node[\nodetype] (\nodename) at (\xpos, \ypos) {};
    \end{scope}
    \begin{scope} 
      \path (8) edge[conn] (22);
      \path (19) edge[conn] (28);
      \path (0) edge[conn] (2);
      \path (12) edge[conn] (18);
      \path (10) edge[conn] (20);
      \path (0) edge[conn] (3);
      \path (1) edge[conn] (7);
      \path (3) edge[conn] (14);
      \path (7) edge[conn] (15);
      \path (22) edge[conn] (30);
      \path (26) edge[conn] (31);
      \path (6) edge[conn] (17);
      \path (3) edge[conn] (8);
      \path (19) edge[conn] (27);
      \path (4) edge[conn] (11);
      \path (4) edge[conn] (9);
      \path (24) edge[conn] (30);
      \path (2) edge[conn] (8);
      \path (16) edge[conn] (25);
      \path (18) edge[conn] (29);
      \path (12) edge[conn] (25);
      \path (6) edge[conn] (15);
      \path (9) edge[conn] (17);
      \path (16) edge[conn] (23);
      \path (5) edge[conn] (10);
      \path (10) edge[conn] (23);
      \path (0) edge[conn] (1);
      \path (22) edge[conn] (27);
      \path (4) edge[conn] (16);
      \path (5) edge[conn] (13);
      \path (24) edge[conn] (28);
      \path (25) edge[conn] (30);
      \path (30) edge[conn] (31);
      \path (14) edge[conn] (25);
      \path (11) edge[conn] (24);
      \path (16) edge[conn] (24);
      \path (23) edge[conn] (28);
      \path (20) edge[conn] (27);
      \path (8) edge[conn] (15);
      \path (0) edge[conn] (4);
      \path (0) edge[conn] (5);
      \path (12) edge[conn] (21);
      \path (4) edge[conn] (12);
      \path (23) edge[conn] (29);
      \path (10) edge[conn] (19);
      \path (3) edge[conn] (7);
      \path (14) edge[conn] (22);
      \path (27) edge[conn] (31);
      \path (7) edge[conn] (18);
      \path (21) edge[conn] (30);
      \path (15) edge[conn] (27);
      \path (13) edge[conn] (22);
      \path (5) edge[conn] (16);
      \path (6) edge[conn] (19);
      \path (15) edge[conn] (26);
      \path (13) edge[conn] (24);
      \path (29) edge[conn] (31);
      \path (2) edge[conn] (6);
      \path (1) edge[conn] (10);
      \path (28) edge[conn] (31);
      \path (25) edge[conn] (29);
      \path (5) edge[conn] (14);
      \path (13) edge[conn] (19);
      \path (2) edge[conn] (11);
      \path (2) edge[conn] (13);
      \path (7) edge[conn] (20);
      \path (18) edge[conn] (26);
      \path (14) edge[conn] (20);
      \path (8) edge[conn] (21);
      \path (20) edge[conn] (29);
      \path (9) edge[conn] (18);
      \path (17) edge[conn] (26);
      \path (1) edge[conn] (6);
      \path (11) edge[conn] (17);
      \path (21) edge[conn] (26);
      \path (9) edge[conn] (23);
      \path (11) edge[conn] (21);
      \path (1) edge[conn] (9);
      \path (17) edge[conn] (28);
      \path (3) edge[conn] (12);
    \end{scope}
    \begin{scope} 
      \foreach \nodename/\labelpos/\labelopts/\labelcontent in {%
        1/below//{Lamb meat},
        1/above/line3/{Meat(grouped)},
        1/above//{Ginger},
        1/above/line2/{Cinnamon, Mugwort},
        2/below//{Stew},
        3/below/lines/{Dark Sauce},
        3/above/lines/{Sauce(Grouped)},
        4/below//{Pottage},
        4/above//{Saffron, Marjoram},
        5/below//{Dip with Herbs},
        5/above//{Various(Grouped)},
        6/above//{Savory},
        8/above//{Juniper Berries},
        10/above//{Curry},
        11/above//{Vegetable(Grouped)},
        13/above//{Basil, Cilantro},
        18/above//{Allspice},
        19/above//{Tarragon},
        21/above//{Bay Leaves, Nutmeg},
        26/above//{Thyme},
        27/above//{Pepper Black},
        28/above//{Oregano},
        29/above//{Caraway},
        30/above//{Paprika, Paprika sweet},
        30/above/line2/{Pepper white},
        31/above//{Garlic, Cayenne pepper}
      } \coordinate[label={[\labelopts]\labelpos:{\labelcontent}}](c) at (\nodename);
    \end{scope}
  \end{scope}
\end{tikzpicture}} 
    \caption{The concept lattice diagrams of the $2,18$-core (top) and the
      $14,1$-core (bottom) of the spice data set.}
\end{figure}
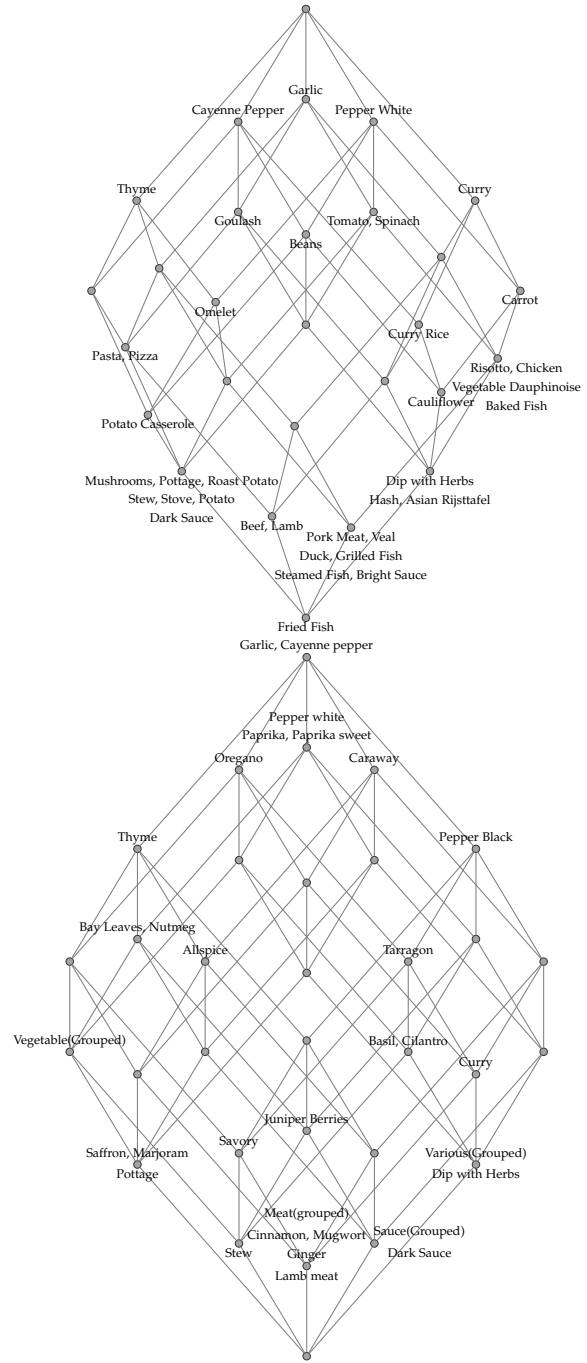

\section{The Problem of Large Contexts}
\label{larger}
Large formal contexts constitute a infeasible problem for classical formal
concept analysis. This is in particular true when computing implicational
theories of them. Applying FCA notions only to \pkcores may be a possible
resort. However, this results in a large number of \pkcores to be considered,
which constitutes a problem in its own, see~\cref{table:datasets}.  Since our
ultimate goal in this work is to present a novel method for coping with large
formal contexts, we demonstrate and evaluate an approach for reducing the search
space for $p$ and $q$ in this section.
For $\Scon\pkleq{p,q}\context$ we know from~\cref{prop:embed} that
$|\mathfrak{B}(\Scon)|$ decreases monotonously in $p$ and $q$.  Let $\hat
p\in\mathbb{N}$ be the maximal number such that for all
$\Scon\pkleq{p,q}\context$ with $p\geq q$ and $|\mathfrak{B}(\Scon)|\leq 30$ we
have that $\Scon\pkleq{p,q}\Tcon\pkleq{\hat p,1}\context$. Furthermore, let
$\hat q\in\mathbb{N}$ be the maximal number such that for all
$\Scon\pkleq{p,q}\context$ with $p<q$ and $|\mathfrak{B}(\Scon)|\leq 30$ we have
that $\Scon\pkleq{p,q}\Tcon\pkleq{1,\hat q}\context$.  This implies that cores
with human readable sized concept lattices are sub-contexts of particular
object- and attribute cores. Our computational approach now is based on finding
those particular cores. Equipped with these contexts we do only need to consider
\pkcores $\Scon\pkleq{p,q}\context$ that are sub-contexts of $\Tcon\pkleq{\hat
  p,1}\context$ or $\Tcon\pkleq{1,\hat q}\context$. Since a direct computation
of $\hat p$ and $\hat q$ is infeasible we suggest an estimation. A naïve
solution for this would be to examine the derivation size distribution of all
objects or attributes. For the data sets investigated in this work this approach
was unsuccessful. More fruitful is a binary search among the parameters. We set
for this the bound for the concept lattice size to 60 as threshold (which is
twice as large as what we consider as readable). Therefore, even if the $\hat
p,1$-core is not human readable, we may encounter $\hat p,q$-core with $q > 1$
that is readable. A general observation for large formal contexts in the
following experiments is that cores with readable concept lattice tend to having
extreme values for parameters $p,q$, i.e., either $p\ll q$ or $q\ll p$.

\subsubsection*{Binary Search For Cores In Mushroom}\label{experiments}
\begin{figure}
  \begin{center}
    \includegraphics[trim=0ex 2ex 1ex 2ex, clip,scale=0.45]{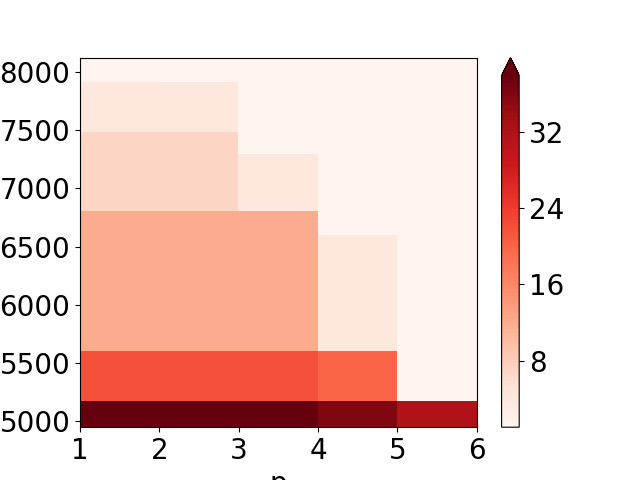}
    \scalebox{1}{\colorlet{mivertexcolor}{black!80}
\colorlet{jivertexcolor}{black!80}
\colorlet{vertexcolor}{black!80}
\colorlet{bordercolor}{black!80}
\colorlet{linecolor}{gray}
\tikzset{vertexbase/.style={semithick, shape=circle, inner sep=2pt, outer sep=0pt, draw=bordercolor},%
  vertex/.style={vertexbase, fill=vertexcolor!45},%
  mivertex/.style={vertexbase, fill=mivertexcolor!45},%
  jivertex/.style={vertexbase, fill=jivertexcolor!45},%
  divertex/.style={vertexbase, top color=mivertexcolor!45, bottom color=jivertexcolor!45},%
  conn/.style={-, thick, color=linecolor}%
}
\tikzstyle{line2} = [label distance=0.25cm]
\tikzstyle{lines} = [xshift=0.7cm,yshift=-0.1cm]
\tikzstyle{lines2} = [xshift=0.7cm,label distance=0.25cm]
\begin{tikzpicture}[scale=0.17]
  \begin{scope} 
    \begin{scope} 
      \foreach \nodename/\nodetype/\xpos/\ypos in {%
        0/vertex/0/0,
        1/jivertex/0/8,
        2/jivertex/-6/10,
        3/jivertex/6/10,
        4/jivertex/-11/13,
        5/jivertex/15/17,
        6/vertex/6/18,
        7/vertex/0/20,
        8/vertex/-10/22,
        9/vertex/-5/23,
        10/vertex/-15/25,
        11/vertex/15/25,
        12/vertex/-21/27,
        13/vertex/9/27,
        14/vertex/21/27,
        15/vertex/4/30,
        16/vertex/-4/32,
        17/vertex/-21/35,
        18/vertex/-9/35,
        19/vertex/21/35,
        20/vertex/-15/37,
        21/vertex/15/37,
        22/vertex/5/39,
        23/vertex/10/40,
        24/vertex/0/42,
        25/vertex/-6/44,
        26/mivertex/-15/45,
        27/mivertex/11/49,
        28/mivertex/-6/52,
        29/mivertex/6/52,
        30/mivertex/0/54,
        31/vertex/0/62
      } \node[\nodetype] (\nodename) at (\xpos, \ypos) {};
    \end{scope}
    \begin{scope} 
      \path (10) edge[conn] (24);
      \path (25) edge[conn] (30);
      \path (2) edge[conn] (12);
      \path (1) edge[conn] (8);
      \path (21) edge[conn] (30);
      \path (18) edge[conn] (29);
      \path (9) edge[conn] (20);
      \path (5) edge[conn] (14);
      \path (0) edge[conn] (5);
      \path (26) edge[conn] (31);
      \path (2) edge[conn] (10);
      \path (18) edge[conn] (26);
      \path (13) edge[conn] (24);
      \path (27) edge[conn] (31);
      \path (28) edge[conn] (31);
      \path (3) edge[conn] (9);
      \path (16) edge[conn] (26);
      \path (23) edge[conn] (30);
      \path (20) edge[conn] (26);
      \path (22) edge[conn] (27);
      \path (5) edge[conn] (13);
      \path (5) edge[conn] (15);
      \path (14) edge[conn] (21);
      \path (3) edge[conn] (6);
      \path (22) edge[conn] (28);
      \path (13) edge[conn] (21);
      \path (6) edge[conn] (16);
      \path (30) edge[conn] (31);
      \path (10) edge[conn] (18);
      \path (4) edge[conn] (15);
      \path (2) edge[conn] (13);
      \path (6) edge[conn] (18);
      \path (23) edge[conn] (27);
      \path (0) edge[conn] (2);
      \path (8) edge[conn] (16);
      \path (0) edge[conn] (4);
      \path (19) edge[conn] (29);
      \path (11) edge[conn] (22);
      \path (21) edge[conn] (29);
      \path (0) edge[conn] (3);
      \path (24) edge[conn] (28);
      \path (8) edge[conn] (17);
      \path (7) edge[conn] (18);
      \path (29) edge[conn] (31);
      \path (1) edge[conn] (6);
      \path (16) edge[conn] (27);
      \path (15) edge[conn] (23);
      \path (12) edge[conn] (17);
      \path (14) edge[conn] (19);
      \path (11) edge[conn] (19);
      \path (9) edge[conn] (23);
      \path (7) edge[conn] (21);
      \path (11) edge[conn] (24);
      \path (12) edge[conn] (25);
      \path (13) edge[conn] (25);
      \path (25) edge[conn] (28);
      \path (8) edge[conn] (22);
      \path (0) edge[conn] (1);
      \path (4) edge[conn] (12);
      \path (1) edge[conn] (10);
      \path (17) edge[conn] (26);
      \path (12) edge[conn] (20);
      \path (10) edge[conn] (17);
      \path (19) edge[conn] (27);
      \path (7) edge[conn] (20);
      \path (4) edge[conn] (8);
      \path (20) edge[conn] (30);
      \path (4) edge[conn] (9);
      \path (15) edge[conn] (22);
      \path (9) edge[conn] (16);
      \path (15) edge[conn] (25);
      \path (14) edge[conn] (23);
      \path (2) edge[conn] (7);
      \path (17) edge[conn] (28);
      \path (1) edge[conn] (11);
      \path (3) edge[conn] (14);
      \path (5) edge[conn] (11);
      \path (24) edge[conn] (29);
      \path (3) edge[conn] (7);
      \path (6) edge[conn] (19);
    \end{scope}
    \begin{scope} 
      \foreach \nodename/\labelpos/\labelopts/\labelcontent in {%
        0/below//{2576},  %
        1/below//{384},  
        2/below//{1360},  
        3/below//{296},  
        4/below//{1440},  
        5/below//{192},  
        6/below//{144},  
        8/below//{384},  
        9/below//{34},  
        10/below//{224}, 
        12/below//{896}, 
        26/above//{veil-color:white},
        26/above/line2/{gill-attachment:free},
        27/above//{gill-size:broad},
        28/above/line2/{ring-number:},
        28/above//{one},
        29/above/lines2/{stalk surface},
        29/above/lines/{above ring:smooth},
        30/above/line2/{gill-spacing:close},
        31/above//{veil-type:partial}
      } \coordinate[label={[\labelopts]\labelpos:{\labelcontent}}](c) at (\nodename);
    \end{scope}
  \end{scope}
\end{tikzpicture}}
    \caption{Heat-map for the core concept lattice sizes (above) and the
      concept lattice of the $5,5176$-core of the mushroom context (below).}
  \end{center}
  \label{size-distro-mushroom}
\end{figure}
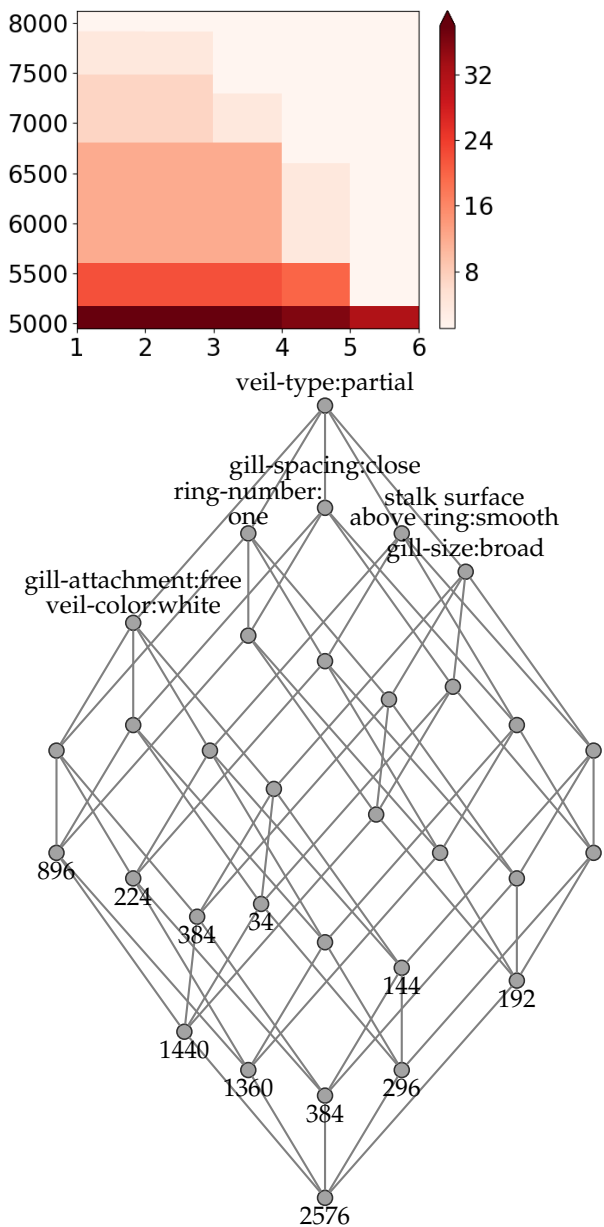

Due to its size (in context as well as in concept lattice terms) the Mushroom
data set is an ideal candidate for the just proposed binary search. Computing
the sizes of all core concept lattices is infeasible. We search as an initial
core for our search paradigm $\hat q$ with $p=1$. We start with $\hat q=|G|$,
which results almost surely in an empty context for real-world data sets. The
binary search in $[1,|G|]$ gives a \pkcore with $p=1$ and $q=4937$. With 38
formal concepts the concept of this sub-context has less than two times 30
concepts, which we considered human readable. Using this core we reduce the
search space to 12832 different $p,q$, which are all bound by 38 in the number
of formal concepts. We may note that searching for some $\hat p$ is impractical
for this data set. This is due to the fact that it was created by scaling
twenty-three non-binary attributes into 119 nominal-scaled attributes. Hence,
there are only two sub-contexts of the mushroom context which are in core
relation for $q=1$. More accurately, these are the mushroom context and the
empty context. We depicted a heat-map of the core concept lattices
in~\cref{size-distro-mushroom} for $q\in[4937,8123]$ and $p\in[1,5]$.  We are
 interested in cores with as much readable conceptual information as
possible, which are cores with $4937<p<5100$, that are also interesting. Out of
those we find the $5,5176$-core is interesting. This core contains seven
distinct attributes and 7930 mushrooms. In the depiction of the corresponding
concept~\cref{size-distro-mushroom} we refrained from annotating all
objects and indicated the number of mushrooms instead (using short-hand
notation from FCA). Hence, to get the total number of objects associated to some
concept one has to add to the object count all numbers from concepts in the
order ideal of that concept. When comparing the core  lattice with the
original lattice we notice that the object number for all concepts with at least
five attributes is similar, which is expected from our theoretical
considerations.

\subsubsection*{Binary Search for Cores in Wiki44k}
To provide another example, we perform the same search in the Wiki44k data
set. The corresponding concept lattice contains 21,923 formal concepts and we
were able to compute that there are approximately 98,000 non-empty \pkcore
contexts. Hence, computing all interesting (\cref{coreChar}) cores is
computationally costly. Therefore, we resort again to the binary search
approach. As the largest attribute core with a readable concept lattice we
identified $1,5202$-core, having 54 formal concepts. We display a heat-map for
the concept lattice size distribution of all sub-cores starting from this bound
in~\cref{wiki1}. As for the object core we discovered that the $15,1$-core has
139 concepts. However, the $16,1$-core is empty, thus we are constrained to
employ the $15,2$-core.  Starting from this we can report that the $15,3$-core
and the $15,4$-core have twenty-five concepts and beyond that the cores are
empty. Hence, those two are interesting candidates. Despite having more concepts
than we considered readable we looked more thorough into the $15,2$-core. Using
background knowledge about the Wikidata properties we are able to present a
well-drawn diagram of its lattice, as depicted in~\cref{wiki1}. We realized that
in this core we do only cover eighteen out of 101 attributes. This is, for
example, in contrast to our observations for the Spice data set, where more than
50\% were covered using a similar sized \pkcorenw. Nonetheless, the $15,2$-core
provides a rough overview about the most important properties in the Wiki44k
data set, in terms of usage for items, and how they are connected.

\begin{figure}
  \scalebox{.2}{\includegraphics{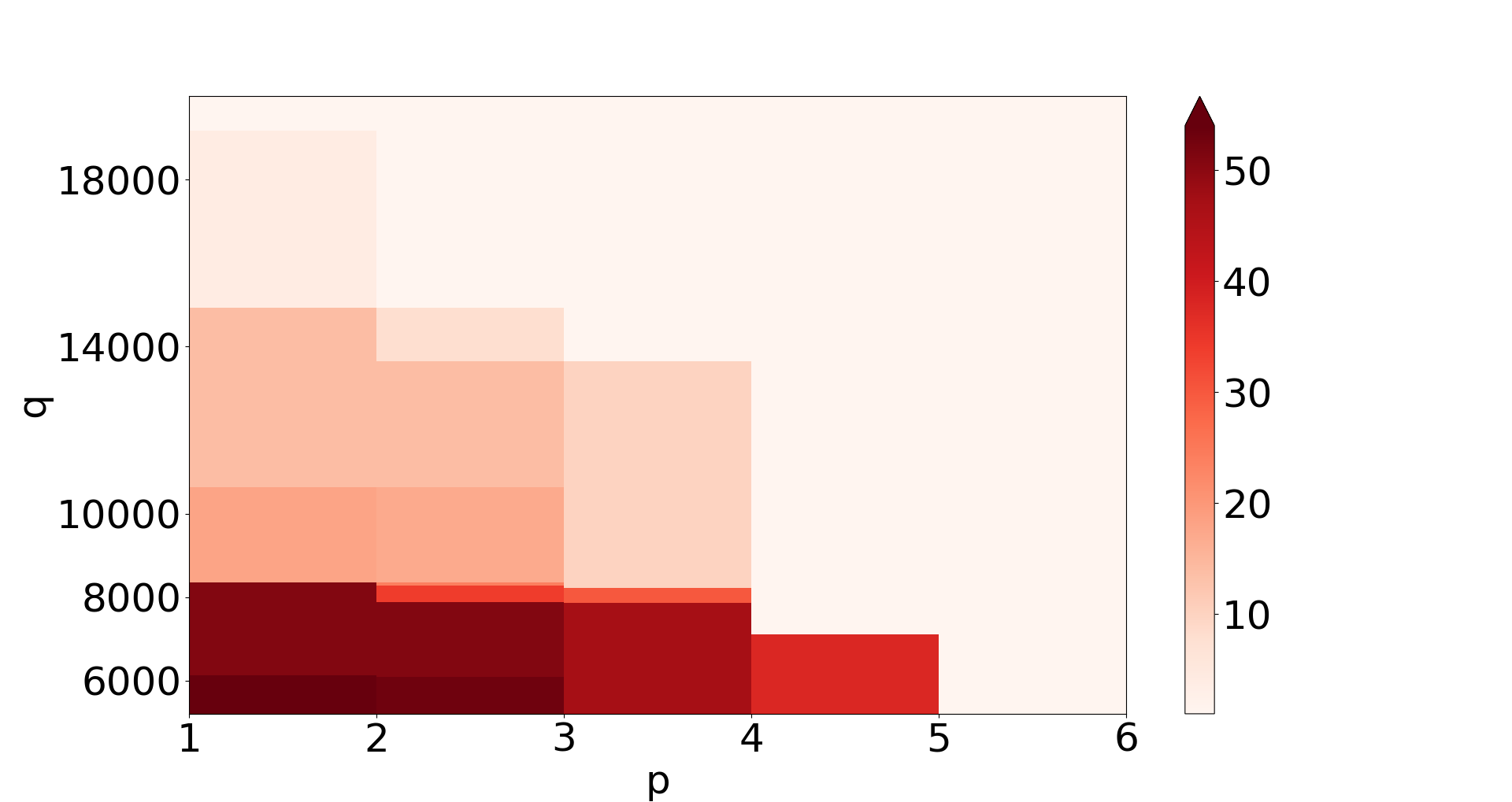}}
  \scalebox{.5}{\colorlet{mivertexcolor}{black!80}
\colorlet{jivertexcolor}{black!80}
\colorlet{vertexcolor}{black!80}
\colorlet{bordercolor}{black!80}
\colorlet{linecolor}{gray}
\tikzset{vertexbase/.style={semithick, shape=circle, inner sep=2pt, outer sep=0pt, draw=bordercolor},%
  vertex/.style={vertexbase, fill=vertexcolor!45},%
  mivertex/.style={vertexbase, fill=mivertexcolor!45},%
  jivertex/.style={vertexbase, fill=jivertexcolor!45},%
  divertex/.style={vertexbase, top color=mivertexcolor!45, bottom color=jivertexcolor!45},%
  conn/.style={-, thick, color=linecolor}%
}
\tikzstyle{line2} = [label distance=0.5cm]
\tikzstyle{lines} = [xshift=-0.7cm,yshift=-0.1cm]
\tikzstyle{lines2} = [yshift=0.1cm]
\tikzstyle{lines3} = [xshift=0.7cm]
\tikzstyle{lines4} = [yshift=-0.2cm]
\begin{tikzpicture}[scale=0.1]
  \begin{scope} 
    \begin{scope} 
      \foreach \nodename/\nodetype/\xpos/\ypos in {%
        0/vertex/92.64738071136499/-33.598201869751264,
        1/jivertex/91.73496876740114/1.4914119685447105,
        2/jivertex/125.56075489784867/4.3545406132275275,
        3/vertex/96.33269413253242/9.79286054447617,
        4/jivertex/56.11196056177543/20.102566539774756,
        5/jivertex/-13.531264052929291/21.835539966206,
        6/jivertex/-23.523033701358173/22.616194954781733,
        7/jivertex/106.97808430452308/27.34665846598648,
        8/vertex/-23.60681130988825/29.569736462778692,
        9/jivertex/20.428848628911314/37.69539185449391,
        10/vertex/20.173419441959577/45.48598205652191,
        11/vertex/10.50848731674862/45.80355215203988,
        12/jivertex/76.59324651242231/45.9132890440605,
        13/mivertex/125.08831412005226/51.75609865213467,
        14/vertex/65.89268390032984/52.52029589859217,
        15/mivertex/10.50848731674862/53.343536919747415,
        16/vertex/76.17435846977189/53.95593946294854,
        17/vertex/46.69078446306169/54.31454683929935,
        18/vertex/65.80890629179979/59.47383740658911,
        19/vertex/38.98324447829398/60.346534653465376,
        20/vertex/45.93678598629094/63.36252856054839,
        21/vertex/-7.32353880549347/69.11652703731919,
        22/mivertex/38.73191165270373/69.22696115765426,
        23/vertex/2.0373191165270264/69.64584920030468,
        24/vertex/82.65494354652999/70.52957216748098,
        25/vertex/72.791419973329/75.78090499307127,
        26/vertex/-7.407316414023544/76.07006854531615,
        27/vertex/91.53537005071888/76.22644954752667,
        28/vertex/82.23605550387957/76.89667041576735,
        29/vertex/40.07235338918508/77.52094440213256,
        30/vertex/1.2217772645750848/77.66184310738774,
        31/vertex/9.912414318354902/78.27494287890332,
        32/vertex/72.70764236479893/82.73444650106823,
        33/vertex/91.36781483365871/83.43132388111388,
        34/vertex/82.57116593799991/83.76643431523422,
        35/mivertex/1.031987814166019/84.0555978674791,
        36/vertex/30.35415079969535/85.0609291698401,
        37/vertex/39.48591012947449/85.0609291698401,
        38/vertex/47.696115765422704/85.22848438690026,
        39/mivertex/81.90094506975925/90.6361982147011,
        40/mivertex/30.270373191165273/92.01447067783705,
        41/mivertex/47.61233815689262/92.93602437166797,
        42/mivertex/38.8994668697639/93.60624523990865,
        43/vertex/38.70967741935483/100.0
      } \node[\nodetype] (\nodename) at (\xpos, \ypos) {};
    \end{scope}
    \begin{scope} 
      \path (24) edge[conn] (29);
      \path (13) edge[conn] (33);
      \path (27) edge[conn] (34);
      \path (17) edge[conn] (19);
      \path (31) edge[conn] (35);
      \path (20) edge[conn] (22);
      \path (23) edge[conn] (26);
      \path (39) edge[conn] (43);
      \path (17) edge[conn] (29);
      \path (5) edge[conn] (10);
      \path (25) edge[conn] (34);
      \path (2) edge[conn] (13);
      \path (3) edge[conn] (30);
      \path (26) edge[conn] (40);
      \path (0) edge[conn] (7);
      \path (6) edge[conn] (21);
      \path (29) edge[conn] (37);
      \path (6) edge[conn] (11);
      \path (7) edge[conn] (16);
      \path (18) edge[conn] (32);
      \path (34) edge[conn] (39);
      \path (27) edge[conn] (33);
      \path (1) edge[conn] (3);
      \path (31) edge[conn] (41);
      \path (26) edge[conn] (35);
      \path (0) edge[conn] (6);
      \path (21) edge[conn] (30);
      \path (7) edge[conn] (13);
      \path (10) edge[conn] (20);
      \path (1) edge[conn] (24);
      \path (12) edge[conn] (16);
      \path (24) edge[conn] (27);
      \path (2) edge[conn] (3);
      \path (20) edge[conn] (37);
      \path (4) edge[conn] (14);
      \path (33) edge[conn] (39);
      \path (14) edge[conn] (18);
      \path (33) edge[conn] (41);
      \path (11) edge[conn] (19);
      \path (19) edge[conn] (22);
      \path (15) edge[conn] (22);
      \path (12) edge[conn] (14);
      \path (24) edge[conn] (28);
      \path (6) edge[conn] (8);
      \path (5) edge[conn] (23);
      \path (0) edge[conn] (9);
      \path (9) edge[conn] (17);
      \path (18) edge[conn] (22);
      \path (4) edge[conn] (11);
      \path (38) edge[conn] (41);
      \path (9) edge[conn] (11);
      \path (38) edge[conn] (42);
      \path (5) edge[conn] (8);
      \path (29) edge[conn] (38);
      \path (11) edge[conn] (15);
      \path (3) edge[conn] (27);
      \path (34) edge[conn] (42);
      \path (36) edge[conn] (42);
      \path (21) edge[conn] (26);
      \path (0) edge[conn] (2);
      \path (21) edge[conn] (36);
      \path (22) edge[conn] (40);
      \path (12) edge[conn] (24);
      \path (35) edge[conn] (43);
      \path (29) edge[conn] (36);
      \path (37) edge[conn] (40);
      \path (17) edge[conn] (20);
      \path (14) edge[conn] (19);
      \path (0) edge[conn] (5);
      \path (27) edge[conn] (38);
      \path (16) edge[conn] (18);
      \path (23) edge[conn] (37);
      \path (32) edge[conn] (40);
      \path (16) edge[conn] (20);
      \path (24) edge[conn] (25);
      \path (14) edge[conn] (25);
      \path (1) edge[conn] (21);
      \path (23) edge[conn] (31);
      \path (42) edge[conn] (43);
      \path (9) edge[conn] (10);
      \path (0) edge[conn] (12);
      \path (30) edge[conn] (42);
      \path (3) edge[conn] (31);
      \path (0) edge[conn] (1);
      \path (36) edge[conn] (40);
      \path (28) edge[conn] (37);
      \path (28) edge[conn] (33);
      \path (37) edge[conn] (41);
      \path (19) edge[conn] (36);
      \path (8) edge[conn] (26);
      \path (41) edge[conn] (43);
      \path (10) edge[conn] (15);
      \path (25) edge[conn] (36);
      \path (40) edge[conn] (43);
      \path (0) edge[conn] (4);
      \path (8) edge[conn] (15);
      \path (25) edge[conn] (32);
      \path (12) edge[conn] (17);
      \path (16) edge[conn] (28);
      \path (32) edge[conn] (39);
      \path (1) edge[conn] (23);
      \path (28) edge[conn] (32);
      \path (30) edge[conn] (35);
    \end{scope}
    \begin{scope} 
      \foreach \nodename/\labelpos/\labelopts/\labelcontent in {%
        1/below//{Q37134},
        2/below//{Q17142},
        4/below//{Q76736},
        5/below/lines4/{Q152480, Q150665},
        6/below//{Q150494},
        7/below//{Q32347},
        9/below//{Q151321, Q151075},
        12/below//{Q58023, Q130734, Q152756, Q183085},
        13/above//{cause of death (P509)},
        15/above//{position held (P39)},
        22/above//{place of burial (P119)},
        35/above//{occupation (P106)},
        39/above//{country of citizenship (P27)},
        40/above/lines/{follows (P155)},
        41/above/lines3/{mother (P25)},
        42/above/lines2/{spouse (P26)},
        43/above//{religion (P140), place of death (P20), sex or gender (P21), followed by (P156), family (P53)},
        43/above/line2/{place of birth (P19), instance of (P31), child (P40), father (P22), award received (P166)}
      } \coordinate[label={[\labelopts]\labelpos:{\labelcontent}}](c) at (\nodename);
    \end{scope}
  \end{scope}
\end{tikzpicture}}  
     \scalebox{.45}{\colorlet{mivertexcolor}{black!80}
\colorlet{jivertexcolor}{red}
\colorlet{vertexcolor}{black!80}
\colorlet{bordercolor}{black!80}
\colorlet{linecolor}{gray}
\tikzset{vertexbase/.style={semithick, shape=circle, inner sep=2pt, outer sep=0pt},%
  vertex/.style={vertexbase, fill=vertexcolor!45},%
  mivertex/.style={vertexbase, fill=mivertexcolor!45},%
  jivertex/.style={vertexbase, fill=jivertexcolor!45,minimum size=.4cm},%
  divertex/.style={vertexbase, top color=mivertexcolor!45, bottom color=jivertexcolor!45},%
  conn/.style={-, thick, color=linecolor}%
}
\tikzstyle{line} = [xshift=-1cm]
\tikzstyle{liner} = [xshift=0.3cm]
\begin{tikzpicture}[scale=0.3,font=\Large]
  \begin{scope} 
    \begin{scope} 
      \foreach \nodename/\nodetype/\xpos/\ypos in {%
        0/vertex/0/0,
        1/vertex/11/13,
        2/vertex/-15/17,
        3/vertex/16/20,
        4/vertex/-18/22,
        5/vertex/7/23,
        6/vertex/14/26,
        7/vertex/18/26,
        8/vertex/-19/27,
        9/vertex/1/27,
        10/vertex/-13/29,
        11/vertex/23/29,
        12/vertex/-24/30,
        13/jivertex/-4/30,
        14/vertex/8/30,
        15/vertex/-19/31,
        16/vertex/-12/32,
        17/vertex/15/33,
        18/vertex/-24/34,
        19/vertex/21/35,
        20/vertex/25/35,
        21/vertex/-19/37,
        22/vertex/9/37,
        23/vertex/-25/39,
        24/vertex/3/39,
        25/vertex/15/39,
        26/vertex/-7/38,
        27/mivertex/-20/42,
        28/mivertex/22/42,
        29/mivertex/16/46,
        30/vertex/-13/47,
        31/mivertex/10/48,
        32/vertex/-1/51,
        33/vertex/0/66
      } \node[\nodetype] (\nodename) at (\xpos, \ypos) {};
    \end{scope}
    \begin{scope} 
      \path (27) edge[conn] (33);
      \path (11) edge[conn] (19);
      \path (11) edge[conn] (20);
      \path (11) edge[conn] (25);
      \path (29) edge[conn] (33);
      \path (19) edge[conn] (31);
      \path (19) edge[conn] (28);
      \path (3) edge[conn] (11);
      \path (3) edge[conn] (6);
      \path (3) edge[conn] (14);
      \path (3) edge[conn] (7);
      \path (17) edge[conn] (28);
      \path (17) edge[conn] (32);
      \path (9) edge[conn] (22);
      \path (9) edge[conn] (26);
      \path (8) edge[conn] (18);
      \path (8) edge[conn] (26);
      \path (8) edge[conn] (15);
      \path (21) edge[conn] (27);
      \path (21) edge[conn] (30);
      \path (22) edge[conn] (29);
      \path (22) edge[conn] (32);
      \path (31) edge[conn] (33);
      \path (28) edge[conn] (33);
      \path (13) edge[conn] (24);
      \path (13) edge[conn] (26);
      \path (12) edge[conn] (21);
      \path (12) edge[conn] (18);
      \path (5) edge[conn] (17);
      \path (5) edge[conn] (26);
      \path (20) edge[conn] (29);
      \path (20) edge[conn] (28);
      \path (16) edge[conn] (27);
      \path (16) edge[conn] (32);
      \path (6) edge[conn] (19);
      \path (6) edge[conn] (17);
      \path (6) edge[conn] (24);
      \path (0) edge[conn] (2);
      \path (0) edge[conn] (1);
      \path (23) edge[conn] (33);
      \path (32) edge[conn] (33);
      \path (2) edge[conn] (13);
      \path (2) edge[conn] (4);
      \path (4) edge[conn] (8);
      \path (4) edge[conn] (12);
      \path (4) edge[conn] (10);
      \path (18) edge[conn] (23);
      \path (18) edge[conn] (30);
      \path (14) edge[conn] (22);
      \path (14) edge[conn] (24);
      \path (14) edge[conn] (25);
      \path (24) edge[conn] (31);
      \path (24) edge[conn] (32);
      \path (30) edge[conn] (33);
      \path (1) edge[conn] (3);
      \path (1) edge[conn] (9);
      \path (1) edge[conn] (13);
      \path (1) edge[conn] (5);
      \path (26) edge[conn] (32);
      \path (26) edge[conn] (30);
      \path (15) edge[conn] (23);
      \path (15) edge[conn] (32);
      \path (10) edge[conn] (21);
      \path (10) edge[conn] (16);
      \path (10) edge[conn] (26);
      \path (7) edge[conn] (17);
      \path (7) edge[conn] (22);
      \path (7) edge[conn] (20);
      \path (25) edge[conn] (29);
      \path (25) edge[conn] (31);
    \end{scope}
    \begin{scope} 
      \foreach \nodename/\labelpos/\labelopts/\labelcontent in {%
        1/below//{1},
        2/below//{1},
        3/below//{6758},
        4/below//{12284},
        5/below//{3},
        6/below//{1327},
        7/below//{1636},
        8/below//{1511},
        9/below//{1},
        10/below//{2659},
        11/below//{2},
        12/below//{19},
        14/below//{207},
        15/below//{1},
        16/below//{73},
        17/below//{1386},
        18/below//{2},
        19/below//{1},
        21/below//{52},
        22/below//{26},
        23/above/line/{Disambiguation P132},
        24/below//{66},
        25/below//{2},
        26/below//{2593},
        27/below/liner/{3},
        27/above//{located in (P131)},
        28/above//{sex or gender (P21)},
        29/above//{country of citizenship (P27)},
        30/below//{241},
        30/above//{country (P17)},
        31/above//{occupation (P106)},
        32/below//{11645},
        32/above//{instance of (P31)}
      } \coordinate[label={[\labelopts]\labelpos:{\labelcontent}}](c) at (\nodename);
    \end{scope}
  \end{scope}
\end{tikzpicture}}
     \caption{The heatmapt of all core concept lattice sizes of the
       $1,5202$-core (above) of the wiki44k data set, the concept
       lattice of its $15,2$-core (middle) and the $1,8290$-core
       (bottom).}
  \label{wiki1}
\end{figure}
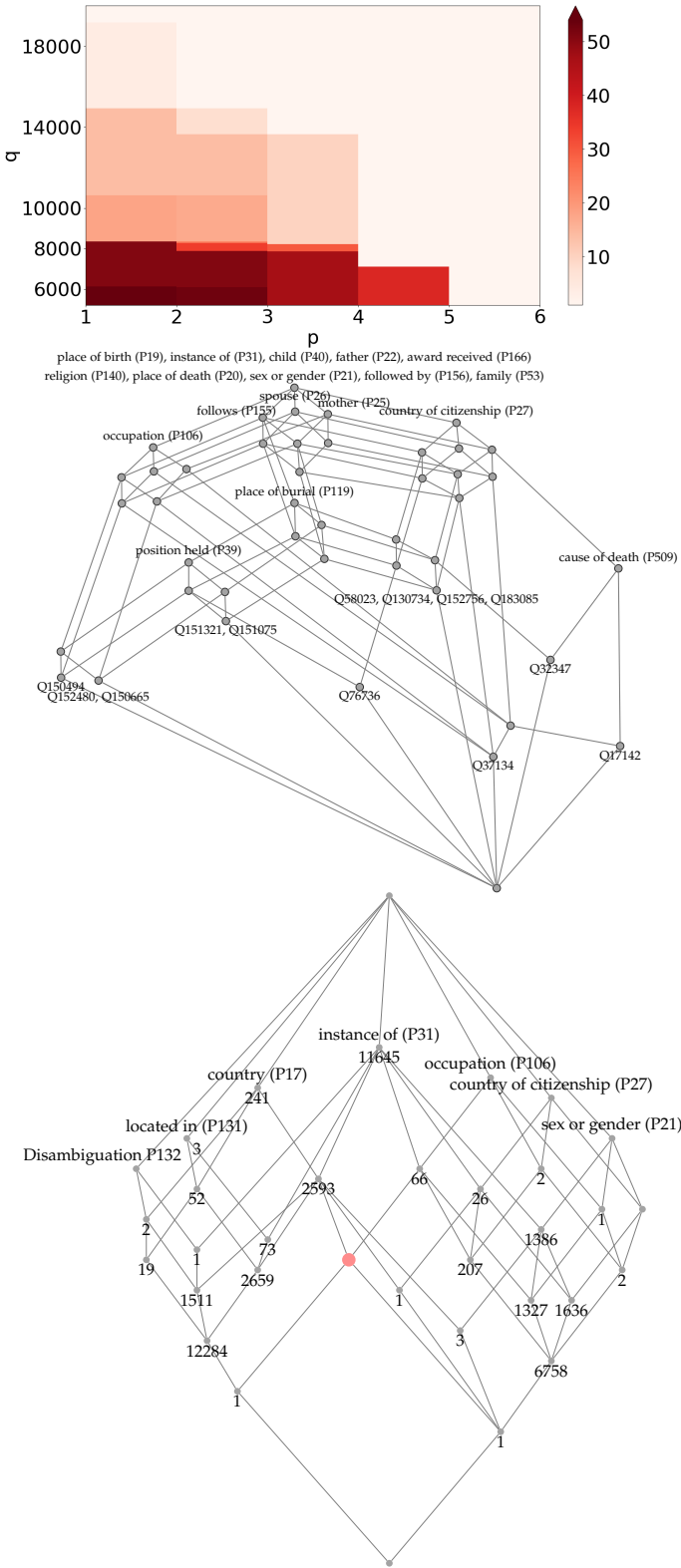

Coming back to the object core investigation, we start with the
$1,5202$-core. From there we find two candidates for interesting
\pkcores, namely the $1,8290$-core on 41735 objects, seven attributes
with 34 concepts and the $4,7115$-core on 20748 objects, eight
attributes with 38 concepts. Although the latter covers more
attributes we decided to look into the former. The reason for this is
the increased readability (due to a lower number of concepts) and the
higher object coverage. Cores with a higher object coverage entail
implications with a higher confidence in the original concept lattice,
see~\cref{Core-propositions}. For the visualizations of~\cref{wiki1}
we decided to indicate the objects using their Wikidata item numbers
instead of their labels. This core describes a majority of the
WikiData entities contained in the dataset. The Wiki44k data set
employs properties used for countries or people for the majority of
statements. Using our proposed core analysis we are able to provide an
human readable diagram representing how these properties are
related. This, in turn, enables us to identify logical errors. For
example, we found that there entities which are countries with an
occupation and a gender, see the concept in~\cref{wiki1} indicated in
red. The Wikidata description of these properties, however, states
that the country property should not be used on human. By a closer
look into the data set we found that one of these entities is "Alfred
A. Knopf", which is both a person (Q61108) and the name of an American
book publisher (Q1431868). Hence, someone added claim to Wikidata on a
wrong item. Besides the study of property usage we can also employ our
analysis method for the identification of missing information, i.e.,
missing statements in Wikidata. We see in~\cref{wiki1} that all
properties that are depicted on the right part of the diagram describe
human features, e.g., occupation (P106), country of citizenship (P27),
and gender (P21). Honoring the constraint that occupation is only to
be used for instances of (P31) human (Q5), we find 66 items having
P106 but missing the property P27. For example, one is "James Blunt"
(Q130799), an English singer-songwriter.

The approach described above can be conducted for arbitrary combinations of
Wikidata properties. Hence, \pkcores enable the user to validate or contradict
reasonable constraints in incomprehensible sized data sets, at least to some
confidence. Furthermore, the \pkcore approach enables an automated procedure for
checking implicational bases, cf.~\cref{Core-propositions}. In particular, one
could employ methods from~\cite{hanika2019discovering} to investigate
implicational bases in Wikidata through pre-computing feasible sized \pkcore
contexts.

\subsection{Comparison with the TITANIC approach}
TITANIC~\cite{stumme2002computing} is an a Apriori based approach that computes
all formal concepts having a minimum support in the data set. Like Apriori,
TITANIC computes these concepts in a bottom-up fashion, with respect to the
attributes. This results in an ordered set of concepts which constitutes a
join-semilattice. An example of such a result, here based on the Mushroom data
set, is depicted in~\cref{core-titanic}. In the following we compare concept
lattices arising from \pkcores to the join-semilattices computed through
TITANIC. We reuse for our analysis the pre-identified interesting $5,5176$-core
$\Scon$ of Mushroom (see~\cref{core-titanic}, above) and indicated
support-values (in $\Scon$) for all object concepts, i.e., for all concepts that
fulfill $({g}^{JJ},{g}^{J})$ for $g\in H$. These numbers are to be read as
follows: the true support value for some concept $c$ is the sum of all support
values of concepts in the order ideal ${\downarrow} c$ from $c$. How support
values of $\Scon$ relate to support values in $\context$ was discussed
in~\cref{Core-propositions}. We observe that $\Scon$ comprises seven attributes
compared to the TITANIC semilattice which has twelve. Both conceptual structures
are built-on thirty-two formal concept. In particular, twenty-one intents of
$\Scon$ are present in the TITANIC semilattice. Hence, the \pkcore data
reduction approach exhibits a different notion for selecting important subsets
of data. Nonetheless, a more thorough investigation of the differences in
applicability to real-world problems is deemed future work.

\begin{figure}[t]
  \begin{center}
    \scalebox{0.6}{\colorlet{mivertexcolor}{black!80}
\colorlet{jivertexcolor}{black!80}
\colorlet{vertexcolor}{black!80}
\colorlet{bordercolor}{black!80}
\colorlet{linecolor}{gray}
\tikzset{vertexbase/.style={semithick, shape=circle, inner sep=2pt, outer sep=0pt, draw=bordercolor},%
  vertex/.style={vertexbase, fill=vertexcolor!45},%
  mivertex/.style={vertexbase, fill=mivertexcolor!45},%
  jivertex/.style={vertexbase, fill=jivertexcolor!45},%
  divertex/.style={vertexbase, top color=mivertexcolor!45, bottom color=jivertexcolor!45},%
  conn/.style={-, thick, color=linecolor}%
}
\tikzstyle{line2} = [label distance=0.25cm]
\tikzstyle{line3} = [label distance=1cm]
\tikzstyle{line4} = [label distance=1.5cm]
\tikzstyle{lines} = [xshift=0.7cm,yshift=-0.1cm]
\tikzstyle{lines2} = [xshift=0.7cm,label distance=0.25cm]
\begin{tikzpicture}[scale=0.2]
  \begin{scope} 
    \begin{scope} 
      \foreach \nodename/\nodetype/\xpos/\ypos in {%
        0/vertex/0/0,
        1/jivertex/0/8,
        2/jivertex/-6/10,
        3/jivertex/6/10,
        4/jivertex/-11/13,
        5/jivertex/15/17,
        6/vertex/6/18,
        7/vertex/0/20,
        8/vertex/-10/22,
        9/vertex/-5/23,
        10/vertex/-15/25,
        11/vertex/15/25,
        12/vertex/-21/27,
        13/vertex/9/27,
        14/vertex/21/27,
        15/vertex/4/30,
        16/vertex/-4/32,
        17/vertex/-21/35,
        18/vertex/-9/35,
        19/vertex/21/35,
        20/vertex/-15/37,
        21/vertex/15/37,
        22/vertex/5/39,
        23/vertex/10/40,
        24/vertex/0/42,
        25/vertex/-6/44,
        26/mivertex/-15/45,
        27/mivertex/11/49,
        28/mivertex/-6/52,
        29/mivertex/6/52,
        30/mivertex/0/54,
        31/vertex/0/62
      } \node[\nodetype] (\nodename) at (\xpos, \ypos) {};
    \end{scope}
    \begin{scope} 
      \path (10) edge[conn] (24);
      \path (25) edge[conn] (30);
      \path (2) edge[conn] (12);
      \path (1) edge[conn] (8);
      \path (21) edge[conn] (30);
      \path (18) edge[conn] (29);
      \path (9) edge[conn] (20);
      \path (5) edge[conn] (14);
      \path (0) edge[conn] (5);
      \path (26) edge[conn] (31);
      \path (2) edge[conn] (10);
      \path (18) edge[conn] (26);
      \path (13) edge[conn] (24);
      \path (27) edge[conn] (31);
      \path (28) edge[conn] (31);
      \path (3) edge[conn] (9);
      \path (16) edge[conn] (26);
      \path (23) edge[conn] (30);
      \path (20) edge[conn] (26);
      \path (22) edge[conn] (27);
      \path (5) edge[conn] (13);
      \path (5) edge[conn] (15);
      \path (14) edge[conn] (21);
      \path (3) edge[conn] (6);
      \path (22) edge[conn] (28);
      \path (13) edge[conn] (21);
      \path (6) edge[conn] (16);
      \path (30) edge[conn] (31);
      \path (10) edge[conn] (18);
      \path (4) edge[conn] (15);
      \path (2) edge[conn] (13);
      \path (6) edge[conn] (18);
      \path (23) edge[conn] (27);
      \path (0) edge[conn] (2);
      \path (8) edge[conn] (16);
      \path (0) edge[conn] (4);
      \path (19) edge[conn] (29);
      \path (11) edge[conn] (22);
      \path (21) edge[conn] (29);
      \path (0) edge[conn] (3);
      \path (24) edge[conn] (28);
      \path (8) edge[conn] (17);
      \path (7) edge[conn] (18);
      \path (29) edge[conn] (31);
      \path (1) edge[conn] (6);
      \path (16) edge[conn] (27);
      \path (15) edge[conn] (23);
      \path (12) edge[conn] (17);
      \path (14) edge[conn] (19);
      \path (11) edge[conn] (19);
      \path (9) edge[conn] (23);
      \path (7) edge[conn] (21);
      \path (11) edge[conn] (24);
      \path (12) edge[conn] (25);
      \path (13) edge[conn] (25);
      \path (25) edge[conn] (28);
      \path (8) edge[conn] (22);
      \path (0) edge[conn] (1);
      \path (4) edge[conn] (12);
      \path (1) edge[conn] (10);
      \path (17) edge[conn] (26);
      \path (12) edge[conn] (20);
      \path (10) edge[conn] (17);
      \path (19) edge[conn] (27);
      \path (7) edge[conn] (20);
      \path (4) edge[conn] (8);
      \path (20) edge[conn] (30);
      \path (4) edge[conn] (9);
      \path (15) edge[conn] (22);
      \path (9) edge[conn] (16);
      \path (15) edge[conn] (25);
      \path (14) edge[conn] (23);
      \path (2) edge[conn] (7);
      \path (17) edge[conn] (28);
      \path (1) edge[conn] (11);
      \path (3) edge[conn] (14);
      \path (5) edge[conn] (11);
      \path (24) edge[conn] (29);
      \path (3) edge[conn] (7);
      \path (6) edge[conn] (19);
    \end{scope}
    \begin{scope} 
      \foreach \nodename/\labelpos/\labelopts/\labelcontent in {%
        0/below//{0.325},  
        1/below//{0.048},  
        2/below//{0.172},  
        3/below//{0.037},  
        4/below//{0.182},  
        5/below//{0.024},  
        6/below//{0.018},  
        8/below//{0.048},  
        9/below//{0.004},  
        10/below//{0.028}, 
        12/below//{0.113}, 
        26/above//{veil-color:white},
        26/above/line2/{gill-attachment:free},
        27/above//{gill-size:broad},
        28/above/line2/{ring-number:},
        28/above//{one},
        29/above/lines2/{stalk surface},
        29/above/lines/{above ring:smooth},
        30/above/line2/{gill-spacing:close},
        31/above//{veil-type:partial}
      } \coordinate[label={[\labelopts]\labelpos:{\labelcontent}}](c) at (\nodename);
    \end{scope}
  \end{scope}
\end{tikzpicture}}  
    \resizebox{0.5\textwidth}{0.5\textwidth}{\includegraphics{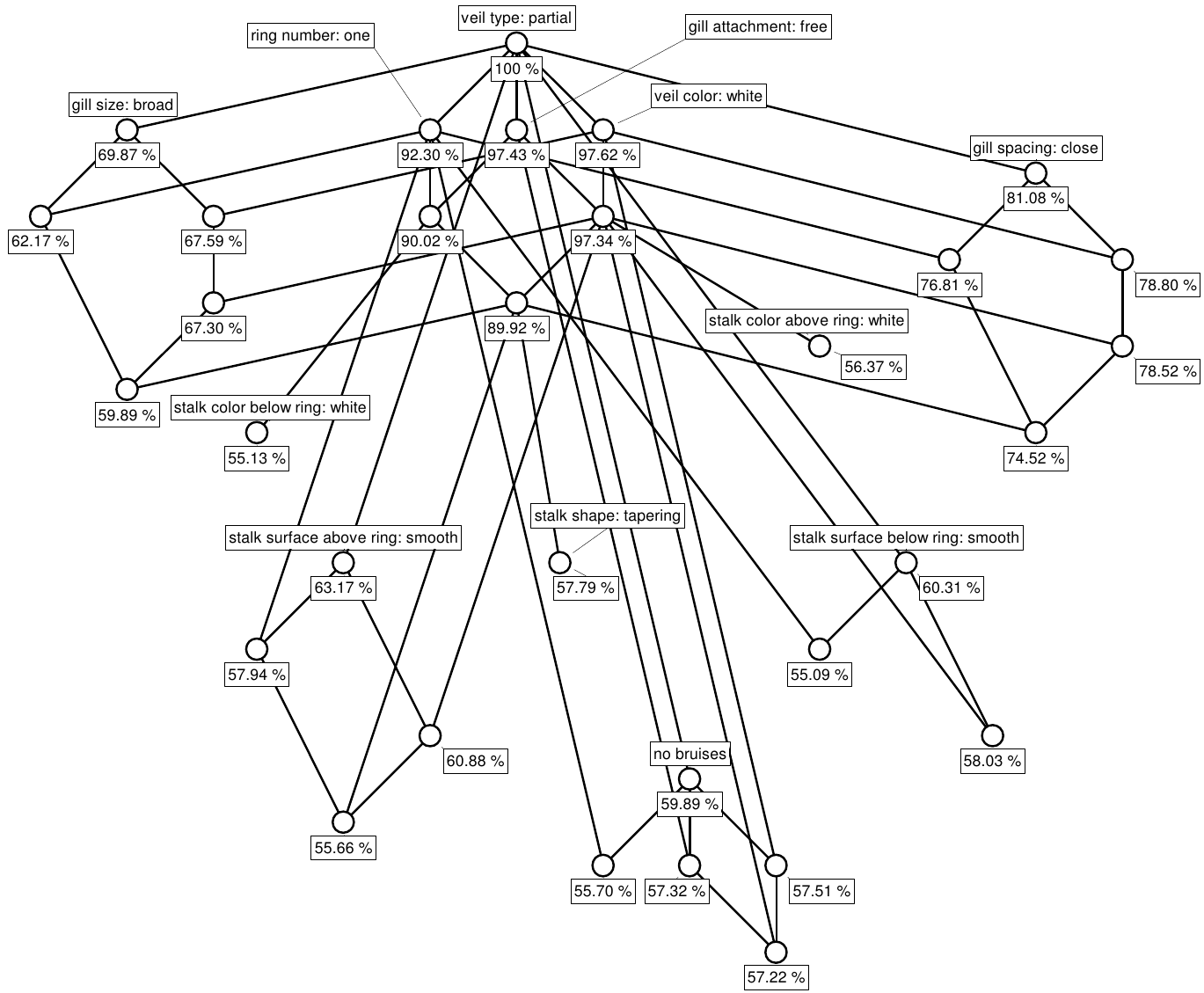}}
  \end{center}
  \caption{The concept lattice of the $5,5176$-core of the mushroom (above) and the TITANIC output with minimum support value
of $55\%$ of the same context (below).}
  \label{core-titanic}
\end{figure}
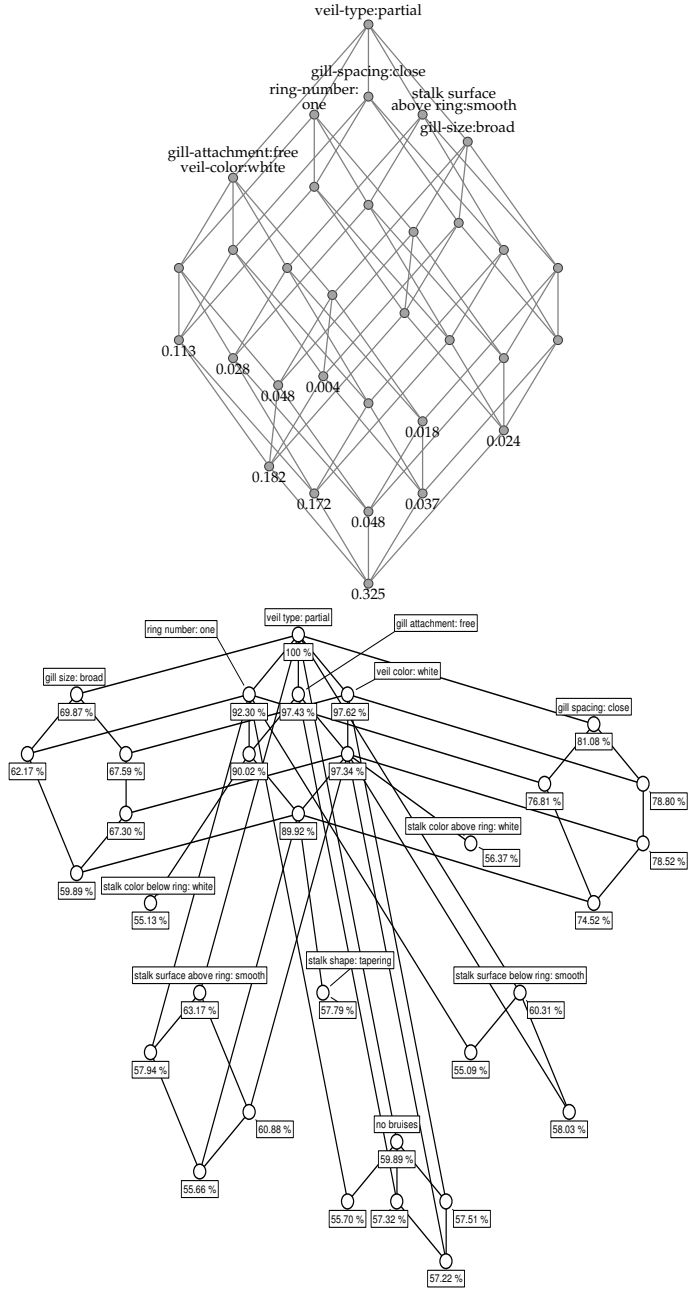

For the rest of this section we investigate the implications one can draw from
the TITANIC semilattice $\mathfrak{T}$ and compare them to the ones valid in the
\pkcore $\Scon$.  We know from~\cref{Core-propositions} that all implications
$A\to B$ with premise length at least five are also valid in the mushroom
context. In the \pkcore we have 70 such implications. However, there are no
non-trivial implications with premise length greater or equal five entailed in
$\mathfrak{T}$. We consider this a major advantage of the novel \pkcore approach
in contrast to TITANIC. As for valid implications in $\Scon$ with premise length
less than five we know from~\cref{Core-propositions} that those are implications
with high confidence in the Mushroom context. The support value in the mushroom
context of such an (valid) implication can also be computed according
to~\cref{Core-propositions}. For example, since $|H|=7930$ and $|G|=8123$ we
know that the valid implication $A\to B$ with $s\coloneqq \sup_{\Scon}(A\to B)$
has in the Mushroom data set at least $7930/8123\cdot s$ support. Due to $\Scon$
being a $5,5176$-core we know that support of $A\to B$ is $s$ in the Mushroom
context if $|A\cup B|\geq 5$. From our analysis we conclude: while the \pkcore
of the Mushroom context does not have as much attributes as the TITANIC
semilattice, it may contain more information in terms of implications.


\section{Algorithms}
\label{algorithms}
For a novel data reduction approach it is essential to have efficient algorithms
available. In this section we present two computational problems concerned with
\pkcores and their algorithmic solution. We start with the fundamental problem
of computing the \pkcore $\Scon$ for a given formal context $\context$. Our
solution to this problem is an adaption of an algorithm by Matula and Beck 1983
\cite{journals/jacm/MatulaB83} for computing $k$-cores of graphs. Given some
graph $G=(V,E)$ with $E\subseteq{V\choose{2}}$ it uses bucket queues to
repeatedly find and remove vertices of small degree. The bucket queue \texttt{Q}
is generated with ${\texttt{Q[k]}\coloneqq\{v \in V \mid \mathop{deg}_{G}(v) =
  k\}}$. After that, the algorithm removes iteratively all vertices in buckets
with index smaller than $k$ and reassigns the remaining vertices to buckets of
corresponding degree. Our adaption to \pkcores employs this algorithm. However,
due to the bipartite nature of our data we provision two bucket queues, for
objects and attributes, respectively. The computational cost for initializing
these bucket queues for a context $(G,M,I)$ is $O(\abs{G}\cdot \abs{M})$. The
worst case cost for one removal iteration on both queues is bound by $O(\abs{G}p
+ \abs{M}q)$. In this particular case the algorithm has to update the remaining
derivation size of at most $p$ attributes for each removed object and $q$
objects for each removed attribute respectively. Hence, the total computation
complexity for our algorithm, as presented in Algorithm~\ref{computeCore}, is
$O(\abs{G}\cdot\abs{M})$. A worst case context is one of
interordinal scale as seen in \cref{worstCore}.

\begin{figure}
  \begin{tikzpicture}
      \node (a) at (0,0)
      {  \scalebox{0.7}{\begin{cxt}
            \cxtName{}
            \att{1}
            \att{2}
            \att{3}
            \att{4}
            \att{5}
            \att{6}
            \obj{xxx...}{1}
            \obj{.xxx..}{2}
            \obj{..xxx.}{3}
            \obj{...xxx}{4}
            \obj{....XX}{\textcolor{red}5}
          \end{cxt}}};
      \node (b) at (2.2,0) {...};
      \node (c) at (3.7,0) 
      {\scalebox{0.7}{\begin{cxt}
            \cxtName{}
            \att{1}
            \att{2}
            \att{3}
            \att{4}
            \att{\textcolor{red}5} 
            \obj{xxx..}{1}
            \obj{.xxx.}{2}
            \obj{..xxX}{3}
            \obj{.....}{}
            \obj{.....}{}
          \end{cxt}}};
      \node (d) at (5.6,0) {...};
      \node (e) at (6,0) 
      {\scalebox{0.7}{\begin{tabular}{|l||}
                        \hline
                        \phantom{1} \\
                        \hline                        
                        \hline
                        \phantom{1} \\
                        \hline
                        \phantom{1} \\
                        \hline
                        \phantom{1} \\
                        \hline
                        \phantom{1} \\
                        \hline
                        \phantom{1} \\
                        \hline
                      \end{tabular}}};
  \end{tikzpicture} 
  \label{worstCore}
  \caption{Example for an worst case instance data set for
    Algorithm~\ref{computeCore}. Computing the $3,2$-core (right) results in a cascading
    sequence of removing either one object or attribute (left, middle) in each
    step.}
\end{figure}
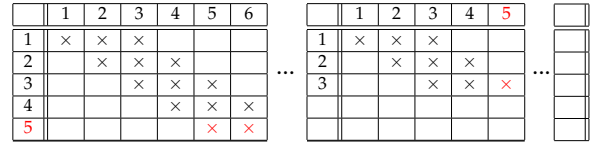

\begin{algorithm}[b]
  \scriptsize
  \caption{Compute $p,q$-core}
  \label{computeCore}
  \DontPrintSemicolon
  \SetKwInOut{Input}{Input}
  \SetKwInOut{Output}{Output}
  \SetKwInOut{Return}{return}

  \Input{ A context $\context = (G,M,I)$ and $p,k \in \mathbb{N}$}
  \Output{\ $\Scon$, with $\Scon \subseteq_{pk} \context$}

  \tcp{initialize core context}
  init output $(H,N,J)$ as $(G,M,I)$\\
  \tcp{initialize bucket lists}
  init $A$, with $A[i] = \{g \in U \mid \abs{g^J} = i \}$\\
  init $B$, with $B[i] = \{m \in V \mid \abs{m^J} = i \}$\\
  \While{$\exists g \in A[i<p]\quad$ or $\quad \exists m \in B[i<k]$}{
    $U = U \setminus \{g \in A[i<p]\}$\\
    $V = V \setminus \{m \in B[i<k]\}$\\
    $J = J \cap U\times V$\\
    update $A$ and $B$
  }
  \Return{\ $\Scon = (H,N,J)$}
\end{algorithm}

\subsubsection*{Navigating Between \pkcore Lattices}\label{algorithm}
In \cref{coreChar} we characterized the interestingness of cores. This required
knowledge about the corresponding concept lattice sizes of \pkcoresnw. However,
every computation of such an concept lattice is (possibly) costly and the number
of these computations is large. For example, we have seen that the Wiki44k data
set has 97,773 non-empty \pkcoresnw. To overcome this issue (to some extent), we
developed an algorithm based on the theory presented in~\cref{transform_int}
(right).

\begin{problem}[Core Lattice]
  \label{prob:1}
  Given $\context$ and the set of all its concepts $\mathfrak{B}(\context)$
  compute for $\Scon \pkleq{p,q} \context$ the set of concepts
  $\mathfrak{B}(\Scon)$.
\end{problem}

For solving this problem we present Algorithm~\ref{transformInt}, which is based
on~\cref{prop:1,prop:3}. This algorithm employs a so for not recollected notion
in FCA, \emph{duality}. We say the dual of a formal context $\Scon=(H,N,J)$ is
$\Scon^{d}\coloneqq (N,H,J^{-1})$. Furthermore, by abuse of notation, we denote
by $\mathfrak{B}(\Scon)^{d}$ the set of concepts of the dual context. The
algorithm solves~\cref{prob:1} in following manner. First, all attributes not in
$\Scon$ are removed by the method \texttt{remove\_attributes}. This is realized
by intersecting all intents with the set~$N$ {(Line~8,~left)}. We construct the
new extent as follows: we compute all extents associated to the same intent,
i.e., intersection with $N$ yields $\intent{c}\cap N$ and form the union of them
(Line 8, right). We justify this using the following lemma.
\begin{lemma}\label{algorithmextent}
  Let $\Tcon = (U,V,L)$ and $\mathbb{S}=(U,N,J)$ with $\mathbb{S}\leq
  \context$. Then we find that  $\forall D \in \Int(\mathbb{S}): D^{J} =
  \bigcup_{B\in\Int(\context),B\cap N=D} B^{I}$.
\end{lemma}
\begin{proof}
  We omit the simple case of $D\in\Int(\Tcon)$ and have therefore
  $D\notin\Int(\Tcon)$.
    \begin{inparaitem}
    \item[$\supseteq$:] Since $D\subseteq B$ it follows that
      $B^{L}\subseteq D^{L}$. We also know that $D^{L}= D^{J}$ because
      of $D\subseteq N$ and the fact that $\mathbb{S}$ is a induced
      sub-context of $\context$ on the same object set. Thus,
      $B^{L}\subseteq D^{J}$.
    \item[$\subseteq$:] For each $D\in \Int(\mathbb{S})$, $D^J\in \Ext(\Tcon)$,
      according to~\cite[Proposition 30]{fca-book}. Therefore, we know that
      $D^{JL}\in \Int(\Scon)$. With $\Scon \leq \context$ we know that
      $D\subseteq D^{JL}$ and following $D^{JL}\cap N = D^{LL}=D$.  Therefore,
      for each $D\in \Int(\Scon)$ there exists a $B\in \Int(\Tcon)$ with $B\cap
      N = D$ and $B^L=D^J$. Hence, $D^J\subseteq \bigcup_{B\in \Int(\Tcon),
        B\cap N=D} B^{L}$.
    \end{inparaitem}
\end{proof}
\noindent 
Secondly, we remove all objects that are not contained in $\Scon$ from the
extents of $\mathcal{B}$ and apply the same \texttt{remove\_attributes} method
to the duals (see Line 4).

The overall run-time complexity of this algorithm is linear in the number of
concepts, since the computation of duals is linear and the overall iteration
consumes the set of concepts. This is an improvement compared to the output
polynomial time complexity of the common computation of $\mathfrak{B}(\Scon)$.

In case we only require to compute the set of all concept intents of a
\pkcorenw, we can apply~\cref{prop:3} in combination with the cover relation of
the concept lattice. This relation of $(\mathfrak{B}(\context),\leq)$ is given
by $\prec\subseteq \leq$ such that for all $c,d\in\mathfrak{B}(\context)$ we
have $c\prec d$ iff $c<d$ and there is no $e\in\mathfrak{B}(\context)$ with $c <
e < d$. Using both~\cref{prop:3} and $\prec$ we can remove all attributes
through intersecting with $N$ (cf. Algorithm~\ref{transformInt}). Afterwards it
is sufficient to remove meet-irreducible intents with cardinality $<p$. These
can be identified easily using the cover relation, i.e., the elements with
exactly one upper neighbor.

\begin{algorithm}[t]
  \scriptsize
  \caption{Transform Core Concepts}
  \label{transformInt}
  \DontPrintSemicolon

  \SetKwInOut{Input}{Input}
  \SetKwInOut{Output}{Output}
  \SetKwInOut{Return}{return}
  \Input{
    $\Tcon = (U,V,L)$ and $\Scon = (H,N,J)$\\
    with $\Scon\pkleq{p,q}T$ and $\mathfrak{B}(\Tcon)$}
  \Output{$\mathcal{B}(\Scon)$}
  \SetKwProg{Trans}{def \texttt{compute-core-lattice}($\Tcon,\Scon,\mathcal{B}(\Tcon)$)}{:}{}
  \SetKwProg{RemA}{def \texttt{remove\_attributes}($\Tcon,\Scon,\mathfrak{B}$)}{:}{}
    $\Ocon\coloneqq (U,N,L{\cap}U{\times}N)$\\
    $\hat{\mathfrak{B}}=$~\texttt{remove\_attributes}($\Tcon,\Ocon,\mathfrak{B}(\Tcon)$)\\
    \nl
    $\mathfrak{B}(\Scon)=$ \texttt{remove\_attributes}$(\Ocon^{d},\Scon^{d},\hat{\mathfrak{B}}^{d})^{d}$\\
  \RemA{}{
    \tcp{Map $\texttt{M}:\Int(\Tcon)\mapsto \Ext(\Tcon)$}
    Initialize empty map \texttt{M}\tcp*{\texttt{M} returns
      $\emptyset$ for unused keys}
    \For{$c\in \mathcal{B}$}{      
      \texttt{M}[$\intent{c}\cap N$]$\mapsto
      \texttt{M}[\intent{c}]\cup \extent{c}$}
      extract $\mathcal{B}$ from \texttt{M}\\
    \Return{$\mathcal{B}$}
   }

\end{algorithm}

In \cref{transform_int} we illustrate a generalization
of Algorithm~\ref{transformInt} to arbitrary sub-contexts as stated by the
following problem:

\begin{problem}[Lattices of Sub-contexts]\label{problemtransformLattice}
  Let $\Scon = (H,N,J)$ be a formal context and $\mathfrak{B}(\Scon)$ its
  concepts. Compute the set of concepts $\mathfrak{B}(\Tcon)$ of
  $\Tcon=(U,V,L)$, with $L\cap H\times N = J\cap U\times V$.
\end{problem}

\begin{algorithm}[t]
  \scriptsize
  \caption{Transform Concepts}
  \label{transforml}
  \DontPrintSemicolon
  \SetKwInOut{Input}{Input}
  \SetKwInOut{Output}{Output}
  \SetKwInOut{Return}{return}
  \SetKwProg{Prime}{def \texttt{lattice\_transformer}($\context ,\Scon,\Tcon,
    \BV(\Scon) $)}{:}{}
  \SetKwProg{Ins}{def \texttt{insert\_attributes}($\Scon ,\Tcon, \mathfrak{B}$)}{:
    }{}
  \Input{ $\context = (G,M,I)$\\
    $\Scon = (H,N,J)$, induced sub-context of $\context$\\
    $\Tcon = (U,V,L)$, induced sub-context of $\context$\\
    $\BV(\Scon)$}
  \Output{$\BV(\Tcon)$}
    \tcp{Adjust the set of attributes}
    $\Ocon_{a1}=(H,N\cap V,\_),\ \Ocon_{a2}=(H,V,\_)$\\
    $\mathfrak{B}_{a_1} {=} \texttt{remove\_attribute}(\Scon,\Ocon_{a1},\mathfrak{B}(\Scon))$\\
    $\mathfrak{B}_{a_2} {=} \texttt{insert\_attributes}(\Ocon_{a1},\Ocon_{a2},\mathfrak{B}(\Ocon_{a_1})$\\
    \tcp{Adjust the set of objects}
    $\Ocon_{b1}=(H\cap U,V, \_ )$\\
    $\mathfrak{B}_{b_1}^{d} {=} \texttt{remove\_attributes}(\Ocon_{a2}^{d}, \Ocon_{b1}^{d}, \mathfrak{B}_{a_2}^{d})$\\
    $\mathfrak{B}_{b_2}^{d} {=} \texttt{insert\_attributes}(\Ocon_{b1}^{d} ,\Tcon^{d}, \mathfrak{B}_{b1}^{d})$\\
    $\mathfrak{B}(\Tcon) = \mathfrak{B}_{b_2}$\\
  \Ins{}{
  \tcp{init order $\leq$ on attributes $V$ such that}
  \tcp{$\forall m \in V{\setminus} N, \forall n \in N: m\leq n$}
  $\hat{\mathfrak{B}}$ = \texttt{next\_closure} on $\Tcon$ in
  \texttt{lectic}$(\leq)$ starting with $N$ \\
  \For{$(E,I)\in \hat{\mathcal{B}}$}{
      \If{$I{\cap}N$ not closed in $\Tcon$}{
        remove the concept $((I\cap N)^{\Tcon},I{\cap}N)$ form $\mathfrak{B}$
      }
   }
  \Return{$\mathfrak{B}\cup \hat{\mathfrak{B}}$}}
\end{algorithm}

With Algorithm~\ref{transforml} we present an approach
for~\cref{problemtransformLattice}, which is based on~\cref{prop:1,prop:2}.  The
algorithm starts by adapting the intents of $\Scon$ to the attribute set of
$\Tcon$ in two steps. First, attributes not included in $\Tcon$ are removed. For
this we apply the \texttt{remove\_attributes} method of
Algorithm~\ref{transformInt}. Second, to insert missing intents the algorithm
employs the \texttt{insert\_attributes} method which enumerates the set of
missing intents from $\mathfrak{B}(\Tcon)\setminus \mathfrak{B}(\Scon)$. Since
any intent of this set contains at least one element of $V\setminus N$ the
algorithms starts with computing \texttt{next\_closure} of $N$ (see Line 9) in an
pre-chosen order $\leq$ on $V$ such that $\forall m \in V{\setminus} N, \forall
n \in N: m\leq n$. Finally in this step, concepts in
$\mathfrak{B}(\Scon)\setminus \mathfrak{B}(\Tcon)$ need to be removed
(cf.~\cref{prop:2}, ii). Thus, we can perform the removal (see Line 12) using a
simple check (see Line 11). The result $\mathfrak{B}(H,V,\_)$ is then stored as
indicated (see Line 3). The necessary adjustment of the set of objects is performed
in a similar fashion due to duality.

The overall run-time complexity can be estimated by
$O(|\mathfrak{B}(\Scon)|+|\mathfrak{B}(\Tcon)\setminus\mathfrak{B}(\Scon)|\cdot|\Tcon|)$.
This is apparent since the first step is the same as in
Algorithm~\ref{transformInt} and the second step employs one scan of $\Tcon$.
This result enables a fast solution of~\cref{problemtransformLattice}, in
particular in the case of~\pkcores.

\section{Related Work}
\label{relwork}
As FCA is interested in representing knowledge through formal concepts and
knowledge bases, it is computationally demanding. Hence, it is crucial to
develop methods that can compute meaningful reductions of data sets or enable a
computational feasible navigation in them. A popular and simple technique to
achieve this is random sampling from contexts~\cite{journals/ijfcs/RothOK08}. This approach, however, does not allow for a
meaningful control of the result. Moreover, the computed concept lattices do
mostly elude from interpretation or even explanation. Also, another disadvantage
of random sampling of objects and attributes, compared to the proposed \pkcore
method, is that rare attribute combinations are unlikely to be drawn. Yet, these
may represent essential counter-examples for learning a sound propositional Horn
logic of the domain.

Other approaches compress formal contexts with popular machine learning
procedures such as \emph{latent semantic analysis} or unsupervised clustering
algorithms on the object set/ attribute set~\cite{conf/cla/CodocedoTA11,
  journals/eswa/AswanikumarS10}. However, we find the resulting concept lattices
do lack on meaningfulness. Since all mentioned approaches introduces new
attributes, e.g., as linear combination of the original attributes, they often
loose their human explainability. Contrary there are also procedures to
automatically/manually select attributes and objects of relevance to the user
\cite{conf/cla/AndrewsO10,DBLP:conf/iccs/HanikaKS19}. However, these approaches
may require a fair amount of domain knowledge, which is not always
available. Furthermore, such processes are very often time consuming for large
data sets, e.g., with hundreds of attributes, when done manually. A major
shortfall of these techniques is that they do not provide proper estimations for
their impact on the concept lattice of the original data set.

Another course of action to cope with large formal contexts are techniques such
as TITANIC~\cite{stumme2002efficient}. They address the computational and
knowledge size issue by omitting rare attribute combinations, i.e., less
supported ones. We consider this a problem as discussed in the first
paragraph. Nonetheless, an advantage of TITANIC is that the resulting iceberg
’lattice’ is sized comprehensively and does not introduce any error with respect
to the original concept lattice. Nonetheless, when dealing with implicational
knowledge of the investigated domain we can draw less knowledge from iceberg
concept lattices, as observed in~\cref{larger}.

A well-established method for data set reduction originates from the research
field of network analysis, called \emph{cores} \cite{conf/waw/HealyJMA06,
  doerfel2013analysis}. The original idea for this goes back to
Seidman~\cite{Seidman1983}. In there, a network is reduced to a densely
connected part. A variation for bipartite networks are
\pkcoresnw~\cite{conf/apvis/AhmedBFHMM07}. Cores are also applied in the realm
of pattern structures~\cite{conf/www/SoldanoSBBL18}. Our presented work on
\pkcores is based on the research results mentioned in this paragraph and
extends them to knowledge cores in formal contexts. Notions, like the impact of
\pkcores on concept lattices and the canonical bases are so far not
investigated, to the best of our knowledge.

\section{Conclusion}\label{sec:conclusion}
In this work we presented an approach to define and investigate the knowledge
core of a formal context. For this we employed a notion from two-mode networks,
called \pkcoresnw. We transferred the idea from graph theory to formal concept
analysis and introduced the notion of \pkcore formal contexts in a formal
manner. Based on that, we identified essential differences of \pkcore lattices
and their originating concept lattice. In particular we investigated conceptual
differences for general sub-contexts and demonstrated their application to
cores. Secondly, we demonstrated different approaches to data analysis using
\pkcoresnw. Crucial here was the characterization of \emph{interestingness}
among core lattices.

As for practical demonstration we analyzed different data sets. We could show
that our method is able to compute two meaningful core lattices for the
\emph{spices} data set that are also human comprehensible in size. For the
\emph{wiki44k} data set, we were able to pinpoint wrongly used properties as
well as missing information using a core lattice diagram.

Furthermore, we found theoretical results enabling us to depict different
algorithms for computing and transforming core structures from formal context
data sets. As for knowledge bases we were able to provide different estimations
for the validity of implicational knowledge in a concept lattice based on core
concept lattice computations. We notably showed that some transformations can be
done in time linear in the size of the original concept lattice. An
exceptionally interesting result is the now achieved ability to navigate
efficiently between arbitrary core lattices of a data set without recalculating
partially shared concepts.  The more these contexts have in common, with respect
to their closure systems, the faster a transformation will perform.  All
algorithms presented in this work are implemented and provided via the FCA
software~\texttt{conexp-clj}\cite{conf/icfca/HanikaH19}, a free and open-source
research tool written in Clojure.

For future work we identify different meaningful lines of research. First of all
a large experimental study on real-world data sets is required. In such a study
domain experts from different fields should evaluate the meaningfulness of core
knowledge to their research investigations. Second, we envision a combination of
\pkcores with other data reduction approaches. For example, one could couple the
TITANIC approach with \pkcoresnw. In such a setup one could compute an initial
interesting core with our method and employ in a second step TITANIC to compute
an highly supported fraction. In a third research thread we propose a more
thorough investigation of the set of all \pkcoresnw. Although we could show that
this set does not constitute a lattice structure one may draw meaningful
knowledge from investigating the shown order relation with tools from directed
graph analysis. Finally, we anticipate an application of \pkcores in temporal
knowledge settings. Due to the shown efficient adaptability to small changes in
objects or attributes \pkcores are an ideal candidate to maintain the dynamic
knowledge of a domain.


\ifCLASSOPTIONcompsoc
  \section*{Acknowledgments}
\else
  \section*{Acknowledgment}
  \fi This work was funded by the German Federal Ministry of Education and
  Research (BMBF) in its program ``CIDA - Computational Intelligence \& Data
  Analytics'' under grant number 01IS17057.

\ifCLASSOPTIONcaptionsoff
  \newpage
\fi



%

\sloppy
\printbibliography

\end{document}